\documentclass[letterpaper]{article} 
\usepackage{aaai25}  
\usepackage{times}  
\usepackage{helvet}  
\usepackage{courier}  
\usepackage[hyphens]{url}  
\usepackage{graphicx} 
\usepackage{natbib}  
\usepackage{caption} 
\usepackage{algorithm}
\usepackage{algorithmic}

\usepackage{newfloat}
\usepackage{listings}
\DeclareCaptionStyle{ruled}{labelfont=normalfont,labelsep=colon,strut=off} 
\floatstyle{ruled}
\newfloat{listing}{tb}{lst}{}
\floatname{listing}{Listing}
\usepackage{amsmath,amssymb,amsthm}
\newtheorem{lemma}{Lemma}[section]

\newtheorem{lem}[lemma]{Lemma}

\newtheorem{definition}[lemma]{Definition}
\newtheorem{theorem}[lemma]{Theorem}

\newtheorem{problem}[lemma]{Problem}
\newtheorem{con}[lemma]{Conjecture}

\newcommand{\cL}{{\mathcal{L}}}
\newcommand{\cD}{{\mathcal{D}}}

\newcommand{\cY}{{\mathcal{Y}}}

\title{On the Hardness of Training Deep Neural Networks Discretely\footnote{A preliminary version of this paper will appear in the proceedings of AAAI 2025.}}
\author {Ilan Doron-Arad
}
\affiliations{
    Technion – Israel Institute of Technology, Haifa, Israel\\
ilan.d.a.s.d@gmail.com
}

\usepackage{bibentry}

\begin{document}
	
	\maketitle
	
	\begin{abstract}
		We study {\em neural network training (NNT)}: optimizing a neural network's parameters to minimize the training loss over a given dataset. NNT has been studied extensively under theoretic lenses, mainly on two-layer networks with linear or ReLU activation functions where the parameters can take any real value (here referred to as {\em continuous} NNT (C-NNT)). However, less is known about deeper neural networks, which exhibit substantially stronger capabilities in practice. In addition, the complexity of the {\em discrete} variant of the problem (D-NNT in short), in which the parameters are taken from a given finite set of options, has remained less explored despite its theoretical and practical significance.

		In this work, we show that the hardness of NNT is dramatically affected by the network depth. 
		Specifically, we show that, under standard complexity assumptions, D-NNT is not in the complexity class NP even for instances with fixed dimensions and dataset size, having a deep architecture. This separates D-NNT from any NP-complete problem.  
		Furthermore, using a polynomial reduction we show that the above result also holds for C-NNT, albeit with more structured instances. We complement these results with a comprehensive list of NP-hardness lower bounds for D-NNT on two-layer networks, showing that fixing the number of dimensions, the dataset size, or the number of neurons in the hidden layer leaves the problem challenging. Finally, we obtain a pseudo-polynomial algorithm for D-NNT on a two-layer network with a fixed dataset size.     
	\end{abstract}
	
	%

	\section{Introduction}
	
	Deep neural networks are among the leading tools in machine learning today, with an astonishingly vast range of applications \cite{DBLP:books/daglib/0040158}. Neural networks require training: optimizing the weights and biases on a given dataset. While minimizing the generalization error is the primary goal, a fundamental algorithmic challenge is to minimize the training loss on the given dataset. 
	
	Neural network training has been studied extensively from a theoretical perspective. NP-hardness proofs have been known for over three decades  \cite{judd1988complexity,blum1988training,megiddo1988complexity} and recent developments obtained strong hardness results for various special cases of the training problem (e.g., \cite{Goel21, Boob22, BertschingerHJM23,Froese2024}). For the above results, it sufficed to consider two-layer networks which already capture the hardness of the problem. 
	
	On the other hand, less is known about deeper neural networks, which adhere to substantially stronger capabilities in practice \cite{DBLP:books/daglib/0040158}. Obtaining lower bounds of training as a function of the network depth is considered a prominent question in previous works (e.g., \cite{Goel21,Froese2024}). In addition, the complexity of the {\em discrete} variant of the problem, in which the parameters are taken from a given finite set of options, has remained less explored despite its theoretical and practical significance.    
	
	In this paper, we aim to narrow the gaps described above, by theoretically studying the following questions. (i) In polynomially solvable settings for shallow networks (e.g., fixed dimension and dataset size, etc.), how difficult is training on deeper networks? (ii) What are the complexity differences between training on discrete versus continuous parameter spaces? (iii) In particular, can we obtain a hardness results stronger than NP-completeness for training on a discrete parameter space? such results exist for two-layer networks in the continuous parameter space \cite{BertschingerHJM23, NEURIPS2021_9813b270}. We start by formally describing the problem.

	\subsection{Neural Network Training (NNT)}  
	We define the decision problem of whether a given neural network can be optimally trained. This problem is occasionally referred to as {\em empirical risk minimization (ERM)} but hereafter will be called {\em neural network training (NNT)}. The definition captures both discrete and continuous parameter spaces with arbitrary activation and loss functions on an arbitrary directed acyclic graph.

	\subsubsection{Input:} We are given a directed acyclic graph $N = (V = \{s\} \cup H \cup T, E)$ (i.e., the network), where $s$, $H$, $T$ are the input vertex (source), the {\em hidden layers} (neurons), and the output vertices, respectively. The {\em depth} of the network is the longest path from the input to an output and a {\em layer} is a maximal subset of vertices with the same distance from the source; the {\em width}, denoted by $k$, is the maximum cardinality of any layer. We also have a dataset \(\mathcal{D} = \{(x_i, y_i)\}_{i=1}^n\), with input \(x_i \in \mathcal{X} \subseteq \mathbb{R}^{d}\) and output \(y_i \in \mathcal{Y} \subseteq \mathbb{R}^{m \cdot |T|}\).
	
	We are also given an activation function $\sigma^v : \mathbb{R} \rightarrow \mathbb{R}$ for each $v \in H \cup T$ and a loss function \(\cL: \mathcal{Y} \times \mathcal{Y} \rightarrow \mathbb{R}\) satisfying $\cL(a,b) = 0$ if and only if $a = b$. We assume that the activation functions and the loss function are polynomially computable.  Additionally, the input consists of weight spaces $W_e[i] \subseteq \mathbb{R}$ and bias spaces $B_e \subseteq \mathbb{R}$, for all $e \in E$ and dimension $i \in \{1,\ldots, d\}$.\footnote{Technically, for edges $e = (u,v) \in E$ where $u \neq s$, there is only one dimension: $W_e = W_e[i]$ $~\forall i \in \{1,\ldots,d\}$.} Let $W_e = (W_e[i])_{i \in \{1,\ldots, d\}}$ and let $\Theta = (W_e, B_e)_{e \in E}$ be the {\em parameter space}. Finally, we are given an error margin parameter $\gamma \in \mathbb{R}$. 
	
	\subsubsection{The Model:}	Given parameters $\theta = (w_e,b_e)_{e \in E} \in \Theta$, that is, $w_e[i] \in W_e[i]$ and $b_e \in B_e$ for all $e \in E$ and $i \in \{1,\ldots,d\}$, the model $f_{\theta}(x)$ is computed inductively given an input $x \in \mathcal{X}$. Define $z^s(x) = x$. Moreover, for each $v \in H \cup T$, the output $z^{v}(x)$ of $v$ is defined by applying the activation function $\sigma^v$ on the weighted sum of outputs of all vertices with an edge directed towards $v$:
	$z^{v}(x) = \sigma^{v}\left( \sum_{e = (u,v) \in E} w^{e} \cdot z^{u}(x) + b^{e}\right). 
	$ The output of the network is \(f_{\theta}(x) = (z^{t}(x))_{t \in T}\). 
	
	\subsubsection{Objective:} Decide if there are parameters \(\theta \in \Theta\) that admit a training loss bounded by the error margin parameter: $\sum_{i=1}^{n} \cL\left(f_{\theta}(x_i), y_i\right) \leq \gamma$. We refer to the well-studied special case of NNT, where $W_e, B_e = \mathbb{R}$ for every edge $e$, as {\em continuous NNT (C-NNT)}; conversely, if the sets $W_e, B_e$ are finite and are explicitly given as part of the input, the problem is called {\em discrete NNT (D-NNT)}. 
	
	\subsubsection{Motivation:} As gradient decent techniques are a form of C-NNT and ubiquitous in practice, the practical motivation for studying C-NNT is apparent. However, there are also settings in which D-NNT possesses practical significance. {\em Quantization} techniques, transforming C-NNT into D-NNT, are very common in practice \cite{courbariaux2015binaryconnect,rastegari2016xnor,zhu2016trained,zhou2018adaptive,rokh2023comprehensive}. These techniques enable to reduce memory usage and energy consumption while increasing the speed of both training and deployment. In some cases, quantization even allows better model generalization.
	
	Other practical settings in which D-NNT is useful include model deployment in low-memory devices \cite{li2018optimal}, initial parameter values for gradient-based optimization \cite{DBLP:conf/iclr/HuXP20}, or even long-term information storage in biological neural networks \cite{hayer2005molecular,miller2005stability} (see more details in \cite{baldassi2015max}). 
	
	\subsubsection{D-NNT versus C-NNT:} One key difference between D-NNT and C-NNT is the over-parametrized setting (where $k > n$). While a C-NNT instance can always optimally fit the dataset in the over-parametrized setting \cite{zhang2021understanding}, this does not always hold for D-NNT instances. As a naive example, consider a two-layer D-NNT instance with $k=2$ neurons and one data point $(x=1,y=2)$. If the parameter space of the instance contains only zeros, it is impossible to construct a model that optimally fits the dataset. This distinction enables lower bounds and non-trivial algorithms for D-NNT in the over-parametrized setting.

	\subsection{Our Results}
	
	\subsubsection{Hardness of NNT on Deep Networks}
	
	In the following, we partially answer questions (i)-(iii). As mentioned above, there are lower bounds stronger than NP-completeness for NNT on a two-layer network \cite{BertschingerHJM23, NEURIPS2021_9813b270}. However, the intricate part in these hard C-NNT instances is based on the infinite parameter space and infinite precision of real numbers. In contrast, for analogous D-NNT instances on a two-layer network, verifying the correctness of a solution can be done in polynomial time since the entire parameter space is given explicitly in the input. Therefore, D-NNT on shallow networks belongs to NP and a strictly stronger lower bound than NP-completeness for D-NNT would require deeper networks.

	We obtain the following lower bound on the running time of D-NNT in {\em deep} networks.  Specifically, we show that D-NNT is unlikely to be in the complexity class NP, making it significantly more challenging than any NP-complete problem (such as, e.g., the closely related Circuit SAT \cite{lu2003circuit,amizadeh2019learning}).

		\begin{theorem}
		\label{thm:SLP}
		 Assume that the \textnormal{Constructive Univariate Radical Conjecture} \cite{10.1145/3510359} is true and assume \textnormal{NP} $\not \subseteq$ \textnormal{BPP}. Then,  \textnormal{D-NNT} $\notin$ \textnormal{NP}, even for instances with input and output dimension $1$ and weights in $\{0,1,-1\}$.   
	\end{theorem}

	Even though our result considers the discrete version of NNT, we give an easy polynomial reduction from D-NNT to C-NNT. This shows that our lower bound for the discrete problem also applies to its continuous counterpart, albeit for instances with a large output dimension.  
	
	\begin{theorem}
		\label{thm:continous}
		There is a polynomial time reduction from \textnormal{D-NNT} to \textnormal{C-NNT}. 
	\end{theorem}

	\subsubsection{D-NNT on Two-Layer Networks} Our above results shed some light on questions (i) and (iii). In the remainder of this section, we focus on question (ii), discussing D-NNT instances on two-layer neural networks. Our first easy result shows that even in a very restricted scenario, D-NNT is NP-hard. This setting incorporates a dataset of size $1$, a regime that is easily solvable in the continuous setting. This gives a another stark distinction between D-NNT and C-NNT.
	
	\begin{theorem}
		\label{thm:SubsetSum}
		Unless \textnormal{P=NP}, there is no polynomial algorithm that solves \textnormal{D-NNT} even if there is only one hidden layer, the dataset is of size $n = 1$, the dimension is $d = 1$, the activation functions are the identity function, and the loss function is the sum of squares. 
	\end{theorem} In the above result, the hardness of D-NNT follows directly from the maximum number allowed for a weight or a bias. Therefore, henceforth the weights and biases are given in unary; this allows assessing the complexity assuming the parameters are relatively small (polynomial in the input size).  In this setting, we obtain the following strong lower bound on the running time even if the weights are encoded in unary.
	
	\begin{theorem}
		\label{thm:ETH}
		Assuming the  \textnormal{Exponential Time Hypothesis (ETH)}, no algorithm  solves \textnormal{D-NNT} in time $f(k) \cdot N^{o \left(\frac{k}{\log k}\right)}$ where $N$ is the encoding size of the instance, $k$ is the width of the network, and $f$ is any computable function. The result holds even for instances with the identity activation function and input dimension $d = 1$.
	\end{theorem} 
	
	While there are ETH-based lower bounds for NNT, we are only aware of such bounds for C-NNT instances with {\em Rectified Linear Unit (ReLU)} activation functions in which the parameter is $d$ rather than $k$ \cite{froese2022computational}. In contrast, in the above result (and also in the next) the activation function used is the identity function. Such a result is unlikely for C-NNT.  
	
	Finally, we obtain a hardness result for instances with a single neuron in the hidden layer (width $k=1$). This result is different from other hardness results with one neuron \cite{dey20,Goel21,Froese22}, since we consider D-NNT rather than C-NNT and our result applies to the identity function as the activation function.  
	
	\begin{theorem}
		\label{thm:SetCover}
		Assuming  \textnormal{P} $\neq$ \textnormal{NP}, there is no polynomial algorithm for \textnormal{D-NNT} even with one input, one output, one neuron in the hidden layer $k = 1$, and the identity activation function. 
	\end{theorem}

	We complement the above lower bounds with a pseudo-polynomial algorithm for D-NNT on a two-layer network with general activation functions (on the hidden layer only) and a general loss function. This algorithm studies the {\em over-parametrized} setting, in which $k > n$. That is, we assume that the number of neurons in the hidden layer is unbounded, the input dimension is unbounded, and the output dimension is one. On the other hand, we assume that the dataset size is a fixed constant. Finally, the running time is pseudo-polynomial in the maximum number that can be computed via the network.

	\begin{theorem}
		\label{thm:algorithm}
		There is an algorithm that decides \textnormal{D-NNT} on a two-layer network with output dimension $1$ in time $\textnormal{poly}(|I|) \cdot M^{O(n_0)}$, where $M$ is the maximum number that can be computed via the network and $n_0$ is the size of the dataset.  
	\end{theorem}
	
	By Theorem~\ref{thm:ETH} and Theorem~\ref{thm:SetCover}, obtaining qualitatively better running time for the above regime is improbable.  This algorithm is intended as a theoretic proof-of-concept and is unlikely to perform well in practice.  
	
	\subsection{Discussion}
	This section describes the implications and limitations of our results and lists some open questions. 
	
	\subsubsection{Hardness of NNT on Deep Networks:} Theorem~\ref{thm:SLP} indicates that even very simplified D-NNT instances on a deep network architecture are strictly harder than any NP-complete problem. This result implies that, unlike other NNT lower bounds, the depth of the network can be isolated as the sole factor contributing to the hardness of training.  
	This is in contrast to other stronger than NP-completeness lower bounds that are applied to two-layer networks \cite{BertschingerHJM23, NEURIPS2021_9813b270}. 
	
	Several questions remain open here. It would be interesting to obtain even stronger hardness results for deep networks exploiting more general instances with larger dimensions and unbounded dataset size. In addition, our reduction to C-NNT (Theorem~\ref{thm:continous}) augments the output dimension; we would like to reproduce the lower bound of Theorem~\ref{thm:SLP} for C-NNT with a smaller number of outputs, which would require a larger dataset and potentially more expressive activation functions.

	\subsubsection{Discrete vs. Continuous NNT:}
	One key distinction between D-NNT and C-NNT, exemplified in our results, is that D-NNT is often intractable even for instances considered to be naive in the continuous parameter space (e.g., as in the over-parametrized setting \cite{zhang2021understanding}). On the other hand, D-NNT instances can be solved by exhaustive enumeration, unlike C-NNT. Despite the above distinctions, the theoretical dynamics between the discrete and continuous settings are not fully understood  and should be a subject to future works. 
	This question is especially interesting due to the large body of work on
	 quantized models \cite{courbariaux2015binaryconnect,rastegari2016xnor,zhu2016trained,zhou2018adaptive,rokh2023comprehensive}. Our results give indications for the hardness of training quantized models in the worst case. However, in practical settings the quantization is not adversarial and can be chosen in various ways. Thus, the lower bounds in our paper do not give a direct implication for practical quantization schemes.

	\subsubsection{Handling Large Numbers:} The crux in the hardness of deep networks presented in Theorem~\ref{thm:SLP} is based on representing fast-increasing numbers in the network. Note that this does not depend on the maximum absolute parameter value, which is $1$ in our reduction. 
	Thus, our results provides additional theoretical angle on numerical instability observed in the area of deep neural networks (e.g., \cite{hochreiter1998vanishing, zheng2016improving, sun2022surprising}). 
	Nevertheless, in practice, numbers used in deep networks do not usually grow as fast as in our lower bound, designed solely to explore the limitations of deep learning in theory. 
	
	It is intriguing to study further the complexity of the problem allowing the running time to pseudo-polynomially depend on the maximum number that can be computed via the network. Indeed, in Theorem~\ref{thm:algorithm} we obtain an algorithm with such a running time in the over-parametrized discrete setting (on two-layer networks). Pseudo-polynomial lower bounds, especially for deeper networks, would be interesting here. In addition, Theorem~\ref{thm:ETH} describes a scenario in which the encoding size is determined by the maximum weight in the instance. It would be interesting to improve the lower bound to $f(k) \cdot N^{\Omega \left(k\right)}$ under the same conditions. 
	
	\subsubsection{Number of Inputs:} This paper studies NNT only on instances with a single source (which can be multidimensional though). Even with one input, NNT is computationally challenging, with various lower bounds (see the previous work section). Nonetheless, it would be interesting to explore NNT with multiple inputs in a simplified environment.

	\subsubsection{Activation Functions:} In Theorem~\ref{thm:SLP}, we use as a proof-of-concept activation functions that perform as division and multiplication operators. Even though these activation functions are non-standard, they can be efficiently encoded and computed, consequently fitting to polynomial reductions. It would be interesting to use standard activation functions such as ReLU or linear activation functions for obtaining lower bounds in deep networks.

	\subsubsection{Approximation Algorithms:} This paper focuses on obtaining exact theoretical lower bounds for NNT. However, in most practical settings approximate solutions may suffice. We remark that for instances with parameter error $\gamma = 0$, e.g.,  Theorem~\ref{thm:SubsetSum} and Theorem~\ref{thm:SetCover}, there cannot be any multiplicative approximation ratio. Thus, the best we can expect is an additive approximation \cite{Goel21}. It remains an open question whether Theorem~\ref{thm:SLP} can be adapted to rule out multiplicative or additive approximation for (relatively easy) deep networks. We remark however that using simple scaling, the lower bound described  in Theorem~\ref{thm:SLP} can be applied for arbitrarily large values of $\gamma$.

	\subsubsection{Connection to Learning:} In practice, the goal of training is to minimize the generalization loss on unseen data. Conversely, this paper focuses on loss minimization on a static dataset. While there are rooted connections between training and learning (e.g., \cite{shalev2014understanding,GoelKKT17,Goel21}), it would be interesting to find non-trivial assumptions on the data distribution and design analogous lower bound to Theorem~\ref{thm:SLP} in a learning perspective.

	\subsubsection{Scaling of Hyper-Parameters}
	The width plays a pivotal role in Theorem~\ref{thm:SubsetSum} and cannot be reduced using our reduction.  On the other hand, in Theorem~\ref{thm:ETH}, the width $k$ roughly expresses the number of constraints and the weight space $n$ describes the size of the alphabet of a {\em 2-constraint satisfaction problem (2-CSP)} instance \cite{marx2007can}. Observe that Theorem~\ref{thm:ETH} focuses on the {\em parameterized} regime where $n \gg k$; however, our reduction can also yield $\Omega\left(n^k\right)$ ETH-based lower bounds for the regime $n=O(1)$. 
	Finally, in Theorem ~\ref{thm:SetCover}, the size of the dataset and the dimension can scale up to a polynomial factor, based on the hardness of Set Cover \cite{raz1997sub, dinur2014analytical}.
	
	\subsection{Previous Work}
	
	Due to the immense body of work on neural network training, we limit this section to purely theoretical results.    
	There has been a long line of work on training neural networks with continuous weights (C-NNT), with NP-hardness proofs for various special cases \cite{shalev2014understanding,blum1988training}. As a fundamental special case, most research focused on (fully connected) networks with two layers leading to a comprehensive understanding of such instances. The parameters $n,k,d$, and $\gamma$ mentioned in the following results are analogous to the parameters in our definition of the NNT problem.

	A significant portion of previous work studied NNT with ReLU activation functions. Under the above restrictions, the problem was proven to be NP-Hard 
	\cite{dey20,Goel21} even for $k = 1$. For the same regime, a stronger lower bound of $n^{\Omega(d)}$ on the running time has been given by \cite{Froese22} assuming ETH (see Conjecture~\ref{ETH}).  For the setting where the dimension $d$ is a fixed constant, \cite{Froese2024} showed NP-hardness with input dimension $d = 2$. \cite{Goel21} Showed that for error margin parameter $\gamma = 0$ the problem is NP-Hard if and only if $k \geq 2$; also, in the paper they give lower bounds for approximation algorithms. For $ k = 2$ and $\gamma = 0$ there is also an NP-Hardness result of \cite{Boob22}.

	Interestingly, even for two-layer networks with ReLU activation functions, the C-NNT problem is $\exists \mathbb{R}$-Complete (the {\em existential theory of the reals}) \cite{BertschingerHJM23, NEURIPS2021_9813b270}. This complexity class, originating in \cite{DBLP:conf/gd/Schaefer09}, is conjectured to be distinct from NP.
	
	On the algorithmic front, \cite{PilanciE20} showed a polynomial time algorithm for fixed dimensions with a regularized objective.  There are also algorithms with exponential running time (as a function of the number of neurons and the output dimension) \cite{AroraBMM18, froese2022computational}. These results interestingly show that under restrictions on the output dimension and depth of the network, the problem is in NP. For the study of linear activation functions, \cite{khalife2024neural} give an algorithm polynomial in the dataset size, but exponential in the network size and the dimension. 
	
	\subsubsection{Organization:} In Section~\ref{sec:preliminaries} we give some preliminary definitions and notations. In Section~\ref{sec:SLP} we give our lower bounds for deep networks and in Section~\ref{sec:continious} our lower bounds for two-layer networks. Section~\ref{sec:algorithm} presents the pseudo-polynomial algorithm. Finally, the proof of Theorem~\ref{thm:continous} is given in Section~\ref{sec:proof1.2}.
	
	\section{Preliminaries}
	\label{sec:preliminaries}
	
	\subsubsection{Notations} Given a number $n \in \mathbb{N}$, let $\textnormal{\textsf{poly}}(n)$ be a polynomially bounded expression of $n$.  Given a vector $x \in \mathbb{R}^d$, for all $i \in \{1,\ldots, d\}$ we use $x[i]$ to denote the $i$-th entry of $x$. Given an instance $I$ of an optimization/decision problem, we use $|I|$ to denote the encoding size of $I$. As the notations of NNT are quite cumbersome, we use the notations $I = (N = (V,E), \mathcal{D}, (\sigma^v)_{v \in V \setminus \{s\}}, \mathcal{L}, \Theta = (W_e, B_e)_{e \in E}, \gamma)$ given in the formal definition in the introduction as the canonical notation for an NNT instance. We sometimes omit the specific declaration of some objects when clear from the context.

	\subsubsection{Complexity}
		The complexity class NP consists of all decision problems for which a given solution (certificate) can be verified in polynomial time by a deterministic Turing machine. In addition, the complexity class P consists of all decision problems that can be decided by a deterministic Turing machine in polynomial time. 
	 {\em Bounded-error probabilistic polynomial time (BPP)} is the complexity class consisting of all decision problems decidable by a probabilistic Turing machine with success probability at least $\frac{2}{3}$ in polynomial time. It is known that P $\subseteq$ BPP and P $\subseteq$ NP; it is often conjectured that P $=$ BPP and P $\neq$ NP .    
	
	A complexity conjecture used in Theorem~\ref{thm:ETH} is the {\em Exponential-Time Hypothesis (ETH)} \cite{impagliazzo2001complexity}. This conjecture claims that there is no sub-exponential algorithm for the 3-SAT problem. Formally,  
	
	\begin{con}
		\label{ETH}
		{\bf  Exponential-Time Hypothesis (ETH) } 
		There is a constant $\beta > 0$ such that there is no algorithm
		that given a \textnormal{3-SAT} formula $\Phi$ with $n$ variables and $m$ clauses 
		can decide whether $\Phi$ is satisfiable in time $O\left(2^{\beta\cdot n}\right)$. 
	\end{con} 
	
	\section{Proof of Theorem~\ref{thm:SLP}}
	\label{sec:SLP}
	
	In this section, we give the proof of Theorem~\ref{thm:SLP}. 
	The proof is based on a reduction from a decision problem involving {\em straight-line programs (SLP)}. 
	
	\begin{definition}
		\label{def:SLP}
		A
		{\em straight-line program (SLP)} $P$ is a sequence of univariate integer polynomials $(a_0, a_1, \ldots , a_{\ell})$ such
		that $a_0 = 1$, $a_1 = x$ and $a_i = a_j \oplus a_k$~ for all $2 \leq i \leq \ell$, where $\oplus \in \{+, -, *\}$ and $j, k < i$. We use $\tau(f)$ to denote the minimum length of an SLP that computes a univariate polynomial $f$.
	\end{definition}

	We use a conjecture proposed in \cite{10.1145/3510359}. A {\em radical} $\textnormal{\textsf{rad}}(f)$ of a non-zero integer polynomial $f \in \mathbb{Z}[x_1, \ldots , x_n]$ is the product of the irreducible
	integer polynomials dividing $f$.

	\begin{con}
		\label{con:radical}
		\textnormal{\textsf{Constructive univariate radical conjecture:}} For any polynomial $f \in \mathbb{Z}[x]$, we have
		$\tau(\textnormal{\textsf{rad}}(f)) \leq \textnormal{\textsf{poly}}(\textnormal{\textsf{rad}}(f))$. Moreover, there is a randomized polynomial time algorithm which, given an \textnormal{SLP} of size
		$\ell$ computing a polynomial $f$, constructs an \textnormal{SLP} for $rad(f)$ of size $\textnormal{\textsf{poly}}(s)$ with success probability at least $1-\frac{1}{\Omega(s^{1+\varepsilon})}$
		for some
		$\varepsilon > 0$.
	\end{con}
	
	Note that integers are a special case of univariate polynomials. Hence, integers can be computed via SLPs. This leads to the following problem originating in \cite{DBLP:journals/siamcomp/AllenderBKM09}. 
	
	\begin{problem}
		\label{prob:SLP}
		\textnormal{\textsf{PosSLP:}} Given an \textnormal{SLP} $P$ computing an integer $n_P$, decide if $n_P > 0$. 
	\end{problem}
	
	In a PosSLP instance, we are given only the compact representation of the SLP. Thus, the variables are not explicitly computed but are defined inductively over the definition of previously defined variables. We will use a strong lower bound on \textnormal{PosSLP} obtained by \cite{DBLP:conf/soda/BurgisserJ24}. 
	
	\begin{theorem}
		\label{thm:PosSLP}
		\cite{DBLP:conf/soda/BurgisserJ24}: If Conjecture~\ref{con:radical} is true and \textnormal{PosSLP} $\in$ \textnormal{BPP}, then \textnormal{NP} $\subseteq$ \textnormal{BPP}.
	\end{theorem}

	We can now prove Theorem~\ref{thm:SLP}. The crux is to use a neural network to efficiently compute the SLP variables, even though the weight space has to be calculated a priori. In the proof, we create $2$ neurons for computing each operation of the SLP. The depth of the neural network will be $\Theta(\ell)$. The parameter space and activation function are efficiently encoded such that the output of each layer is the $i$-th component of the given PosSLP instance. Importantly, the reduced D-NNT instance is {\em restricted}: there is exactly one option for each parameter. 
	
	\begin{lemma}
		\label{lem:reductionSLPtoD-NNT}
		There is a polynomial time reduction that given a \textnormal{PosSLP} $P$ computes a \textnormal{restricted D-NNT} instance $I$ such that $P$ is a \textnormal{YES} instance if and only if $I$ is a \textnormal{YES} instance. 
	\end{lemma}

	\begin{proof}
		Let $P = (a_1, \ldots, a_{\ell})$ be a PosSLP instance. Since in PosSLP we compute an integer, only a special case of a univariate polynomial, we may assume without the loss of generality that we are given an SLP formula $P$ such that $a_0 = 1$ and $a_i = a_j \oplus a_k$ for all $1 \leq i \leq \ell$, where $\oplus \in \{+, -, *\}$ and $j, k < i$. Based on $P$, define a D-NNT instance $I$ as follows. Define two neurons $h_i, h'_i$ for each $0 \leq i \leq \ell$, where $s = h_0$ is the input vertex. Also, define the output vertex as $t = h_{\ell}$ and let $V$ be the set of vertices of the network. For the following, fix some $1 \leq i \leq \ell$ and $j, k < i$ such that $a_i = a_j \oplus a_k$. 
		
		Define the edge $e_{i,j} = (h_{j}, h_i)$. Moreover,
		define the edge $e_{i,k}$ as follows. If $\oplus \in \{+,-\}$, define $e_{i,k} = (h_k, h_i)$;  otherwise (if $\oplus = *$), define $e_{i,k} = (h'_k, h_i)$. 
		Finally, define the edge $e'_{r,r} = (h_r, h'_r)$ for all $0 \leq r \leq \ell$. Let $E$ be the entire set of edges; so far we have defined the network $N = (V,E)$. Define the parameter space $\Theta = (W_e,B_e)_{e \in E}$ as follows. Let $B_{e} = \{0\}$ for all edges $e \in E$ (i.e., the bias of any edge must be zero) and define the weights as follows. Define $W_{e_{i,j}} = \{1\}$, set $W_{e_{i,k}} = \{1\}$ if $\oplus \in \{+, *\}$, and $W_{e_{i,k}} = \{-1\}$ if $\oplus = -$. Finally, let $W_{e'_{r,r}} = \{1\}$ for all $0 \leq r \leq \ell$. Note that the above weights induce a restricted D-NNT instance. 
		
		The activation function $\sigma^{h_i}: \mathbb{R} \rightarrow \mathbb{R}$ is defined by the following cases. If $\oplus \in \{+,-\}$, define $\sigma^{h_i}(\alpha) = \alpha$ for all $\alpha \in \mathbb{R}$ . If $\oplus = *$, then consider the following construction before the definition of $\sigma^{h_i}$. For any $\alpha \in \mathbb{R}$, let $\beta$, $\lambda$ be the unique decomposition of $\alpha$ into an integer $\beta \in \mathbb{Z}$ and $\lambda \in [0,1)$ such that $\alpha = \beta+\lambda$; then, define $\sigma^{h_i}(\alpha) = \beta \cdot \lambda \cdot 10^{\textnormal{\textsf{len}}(\lambda)}$, where $\textnormal{\textsf{len}}(\lambda)$ is the number of decimal digits in $\lambda$ after the decimal point. This function will be applied only to numbers for which $\lambda$ can be represented using a finite number of decimal digits. For example, if $\alpha = 2.55$ then $\beta = 2$, $\lambda = 0.55$, and $\textnormal{\textsf{len}}(\lambda) = 2$; therefore, $\sigma^{h_i}(\alpha) = 2 \cdot 0.55 \cdot 10^2 = 110$. 
		
		Define $\sigma^{h'_i}: \mathbb{R} \rightarrow \mathbb{R}$ such that $\sigma^{h'_i}(\alpha) =  \alpha \cdot 10^{-\textnormal{\textsf{num}}(\alpha)}$, where $\textnormal{\textsf{num}}(\alpha)$ is the total number of digits in $\alpha$ before the decimal point. This function will be applied only to integers. Note that $\sigma^{h'_i}$ converts an integer to a number between $0$ and $1$, for example, $\sigma^{h'_i}(138) = 0.138$. Note that all activation functions can be computed in polynomial time.    
		Define a simple dataset containing merely a single pair $\cD = \{(x,y)\}$, where $x,y = 1$, i.e., both the input and output are one-dimensional.  Let $\cL: \cY \times \cY \rightarrow \mathbb{R}$ be the loss function where $\cL(\alpha,y) = 0$ if $\alpha = 1 = y$, else $\cL(\alpha,y) = 1$ if $\alpha > 0$, and otherwise $\cL(\alpha,y) = 2$. We finally set $\gamma = 1$ as the error margin parameter. The reduction can be computed in time $\textnormal{\textsf{poly}}(\ell)$: there are $O(\ell)$ vertices and edges, the weights are bounded, while the activation function, dataset, as well the loss function, can be computed and encoded in polynomial time and space, respectively. Therefore, the correctness of the lemma follows by showing that for each $0\leq i \leq \ell$ it holds that 
		\begin{equation}
			\label{eq:slp1}
			z^{h_i}(x) = \sigma^{h_i}\left( \sum_{e = (u,h_i) \in E} w^{e} \cdot z^{u}(x) + b^{e}\right) = a_i.
		\end{equation} We prove \eqref{eq:slp1} by induction on $i$. For the base case, let $i = 0$. Then, by definition $z^{h_i}(x) = z^{s}(x) = x = 1 = a_0 = a_i$. Now, let $1 \leq i \leq \ell$ and assume that for all $0\leq r \leq i-1$ it holds that $z^{h_r}(x) = a_r$. By the definition of $a_i$, there are $0 \leq j,k \leq i-1$ such that $a_i = a_j \oplus a_k$ for $\oplus \in \{+,-,*\}$. Then, by the induction hypothesis, it holds that $z^{h_j}(x) = a_j$ and $z^{h_k}(x) = a_k$. We consider three cases depending on $\oplus$. 
		
		\begin{itemize}
			\item $\oplus = +$. Then, it holds that $z^{h_i}(x) = 1 \cdot z^{h_j}(x) +1 \cdot z^{h_k}(x) = a_j \oplus a_k = a_i$. 
			
			\item $\oplus = -$. Then, we have $z^{h_i}(x) = 1 \cdot z^{h_j}(x) -1 \cdot z^{h_k}(x) = a_j \oplus a_k = a_i$. 
			
			\item $\oplus = *$. Then,
			recall that $z^{h_k}(x) = a_k$; thus, $z^{h'_k}(x) = \sigma^{h'_k}(z^{h_k}(x)) = a_k \cdot 10^{-\textnormal{\textsf{num}}(a_k)}$. Note that $a_k$ is an integer (this can be easily proven by induction); thus, $z^{h'_k}(x) \in [0,1)$ and $z^{h'_k}(x)$ can be represented using a finite number of decimal digits.  Therefore, it holds that $\sigma^{h_i}\left(z^{h_j}(x)+z^{h'_k}(x)\right) = a_j \cdot a_k \cdot 10^{-\textnormal{\textsf{num}}(a_k)} \cdot 10^{\textnormal{\textsf{len}}\left(a_k \cdot 10^{-\textnormal{\textsf{num}}(a_k)} \right)} = a_j \cdot a_k = a_i$. The second equality holds by the following. If $a_k = 0$, the equality is immediate; otherwise, since $a_k$ is an integer, the number of digits after the decimal point of $a_k \cdot 10^{-\textnormal{\textsf{num}}(a_k)}$, that is  $\textnormal{\textsf{len}}\left(a_k \cdot 10^{-\textnormal{\textsf{num}}(a_k)} \right)$, is exactly $\textnormal{\textsf{num}}(a_k)$. 
		\end{itemize}
		By the above, the network output is $a_{\ell}$. By the definition of the loss function, there are $\theta = (w_e,b_e)_{e \in E} \in \Theta$ such that $\cL(f_{\theta}(x),y) \leq \gamma$ if and only if $a_{\ell} = n_P > 0$. This gives the statement of the lemma. 
	\end{proof}

	Using the above lemma, we state and prove a lower bound for restricted D-NNT.

	\begin{lem}
		\label{lem:SLPhelperLem}
		Assume that Conjecture~\ref{con:radical} is true and assume that \textnormal{NP} $\not \subseteq$ \textnormal{BPP}. Then, \textnormal{restricted D-NNT} is not in \textnormal{BPP}. 
	\end{lem}
	
	\begin{proof}
		Assume that Conjecture~\ref{con:radical} is true and assume that \textnormal{NP} $\not \subseteq$ \textnormal{BPP}. Using Lemma~\ref{lem:reductionSLPtoD-NNT}, a BPP algorithm for restricted D-NNT implies a BPP algorithm for PosSLP. Hence, \textnormal{restricted D-NNT} is not in \textnormal{BPP}, or we would reach a contradiction to Theorem~\ref{thm:PosSLP}.
	\end{proof}

	From Lemma~\ref{lem:SLPhelperLem} and the definition of the complexity class NP, we can finally prove Theorem~\ref{thm:SLP}.

	\subsubsection{Proof of Theorem~\ref{thm:SLP}:}
	
	Assume that Conjecture~\ref{con:radical} is true and \textnormal{NP} $\not \subseteq$ \textnormal{BPP}. Recall that for every problem in NP, by definition, there exist a polynomial time algorithm that given an instance of the problem and a certificate returns whether the certificate is a solution for the instance. Observe that a restricted D-NNT instance can be viewed as an instance of D-NNT and a certificate consisting of the unique set of parameters of the restricted instance. 
	Hence, by Lemma~\ref{lem:SLPhelperLem}, 
	verifying if the certificate achieves a training loss bounded by the threshold parameter, does not have a BPP algorithm. Since P $\subseteq $ BPP, we conclude that there is no polynomial time that decides if a given set of parameters achieves a training loss bounded by the threshold parameter for a given D-NNT instance. This implies that D-NNT is not in NP by definition (under the assumptions). 
	\qed

	
	%

	%

	\section{Lower Bounds for Two-Layer Networks}
	\label{sec:continious}
	
	In this section, we give the high level ideas in the proofs of our lower bounds for D-NNT on two-layer networks, and provide the full proofs.

	We start by explaining the proof idea of Theorem~\ref{thm:SubsetSum}. The proof is based on a simple reduction from the {\em subset sum} problem, known to be NP-Hard (e.g.,  \cite{books0015106}). In the subset sum problem, we are given a collection $A = \{a_1, \ldots, a_n \in \mathbb{N}\}$ and a target number $T \in \mathbb{N}$. The goal is to decide if there is a subset $S \subseteq A$ such that $\sum_{a \in S} a = T$. 
	
	In the proof, we give a reduction from subset sum to D-NNT on a network with one hidden layer where each edge corresponds to a single item $a_i$ from a subset sum instance. The weights correspond to either taking the item or discarding it. Finally, the loss function guarantees that the overall sum at the output of the network reaches exactly to the target value. 
	
	\subsubsection{Proof of Theorem 1.3:} 
	
	We give below a reduction from subset sum to D-NNT under the restrictions described in the Theorem. Let $\mathcal{S} = (A,T)$ be a subset sum instance, where $A = \{a_1,\ldots, a_n\} \subset \mathbb{N}$ is a set of numbers and $T \in \mathbb{N}$ is a target value. Construct the following D-NNT instance $I = (N = (V = H \cup \{s,t\},E), \mathcal{D}, (\sigma^v)_{v \in V \setminus \{s,t\}}, \mathcal{L}, \Theta = (W_e, B_e)_{e \in E}, \gamma)$. 
	
	\begin{itemize}
		\item The network is a fully connected network with one hidden layer of $n$ vertices. Namely, $N = (H \cup \{s,t\}, E)$ where $H = \{h_1, \ldots, h_n\}$, $s$ is the source of the network, and $t$ is the output vertex. That is, $E = \{(s, h_i), (h_i, t) \mid i \in \{1,\ldots, n\}\}$. 
		
		\item The data set is simply $\cD = \{(x,y) = (1,T)\}$. 
		
		\item For each neuron $h \in H$, the activation function is defined as $\sigma^h(\alpha) = \alpha~\forall \alpha \in \mathbb{R}$ (the identity function). 
		
		\item The loss function is $\cL(\alpha,y) = (\alpha-y)^2~\forall \alpha,y \in \mathbb{R}$ (sum of squares). 
		
		\item For each edge $e_i = (s,h_i)$, $i \in \{1,\ldots,n\}$, define the weight space $W_{e_i} = \{0,a_i\}$ and bias space $B_{e_i} = \{0\}$. 
		
		\item For each edge $e_i = (h_i,t)$, $i \in \{1,\ldots,n\}$, define the weight space $W_{e_i} = \{1\}$ and bias space $B_{e_i} = \{0\}$. 
		
		\item Define $\gamma = 0$ as the error parameter.  
	\end{itemize}
	
	Clearly, the running time of the reduction is polynomial in the encoding size of the subset sum instance $\mathcal{S}$. It remains to prove correctness. Assume that there is a solution $S \subseteq A$ for $\mathcal{S}$. We show that there are parameters $\theta \in \Theta$ such that $\cL(f_{\theta}(x),y) = 0 = \gamma$. Define $w_{e_i} = a_i$ if $a_i \in S$ and $w_{e_i} = 0$ otherwise, for all $e_i = (s,h_i), i \in \{1,\ldots, n\}$. By the definition of the weights, the above uniquely defines a set of weights and biases for the network. Then, by the selection of weights and the definition of the network, it holds that 
	
	\begin{equation*}
		\label{eq:sB2}
		\begin{aligned}
			\sum_{(x,y) \in \cD} \cL(f_{\theta}(x),y) ={} & \cL(f_{\theta}(1),T) 
			\\={} & \cL \left(\sum_{i \in \{1,\ldots, n\}} w_{(s,h_i)} \cdot 1, T \right) \\={} & \left(\sum_{i \in \{1,\ldots, n\}} w_{(s,h_i)} \cdot 1 -T\right)^2 \\={} & \left(\sum_{a \in S} a -T\right)^2 \\={} & 0.
		\end{aligned}
	\end{equation*} The last equality holds since $S$ is a solution for $\mathcal{S}$. 
	
	For the second direction of the reduction, let $\theta = (w_e,b_e)_{e \in E}$ be a set of parameters for the network such that $ \cL(f_{\theta}(x),y) = 0$. Then, by the structure of the network, it implies that 
	
	\begin{equation}
		\label{eq:subsetsum1}
		\begin{aligned}
			0 ={} & \sum_{(x,y) \in \cD} \cL(f_{\theta}(x),y) \\={} & \cL(f_{\theta}(1),T) 
			\\={} & \cL \left(\sum_{i \in \{1,\ldots, n\}} w_{(s,h_i)} \cdot x, T \right) \\={} &		\left(\sum_{i \in \{1,\ldots, n\}} w_{(s,h_i)} \cdot x -T\right)^2 \\={} & \left(\sum_{i \in \{1,\ldots, n\}} w_{(s,h_i)} \cdot 1 -T\right)^2 
			\\={} & \left(\sum_{i \in \{1,\ldots, n\}} w_{(s,h_i)} -T\right)^2 
		\end{aligned}
	\end{equation} 
	
	From here, we define $S = \{a_i \in A \mid w_{(s,h_i)} \neq 0\}$; by \eqref{eq:subsetsum1} it follows that $\sum_{a \in S} a = \sum_{i \in \{1,\ldots, n\}} w_{s,h_i} =T$; thus, $S$ is a solution for $\mathcal{S}$. The above gives the two directions of the reduction. Hence, since subset sum is well-known to be NP-Hard (e.g., \cite{books0015106}), we reached the statement of the theorem. \qed

	We now give the proof of Theorem~\ref{thm:ETH}. The proof is based on a reduction from {\em binary constraint satisfaction problem (2-CSP)}. The input to a 2-CSP instance is a tuple $\Gamma = (G,\Sigma,C)$ such that $G = (V,E)$ is a constraint graph, $\Sigma$ is an alphabet, and $C = \{C_{(u,v)}\}_{(u,v) \in E}$ are constraints such that for all $(u,v) \in E$ it holds that $C_{(u,v)} \subseteq \Sigma \times \Sigma$ and $C_{(u,v)} \neq \emptyset$. An {\em assignment} to $\Gamma$ is a function $\phi: V \rightarrow \Sigma$. An assignment $\phi$ is called {\em feasible} if for all $(u,v) \in E$ it holds that $(\phi(u),\phi(v)) \in C_{(u,v)}$. The goal is to decide if there is a feasible assignment. We will use the following result of \cite{marx2007can}. 
	
	\begin{theorem}
		\label{thm:Chen}
		\cite{marx2007can} Assuming \textnormal{ETH}, for any computable function $f$ there is no algorithm that decides \textnormal{2-CSP} in time $f(k) \cdot N^{o \left(\frac{k}{\log k}\right)}$, where $k$ is the number of constraints and $N$ is the size of the alphabet. 
	\end{theorem}
	
	Using the above, we can now prove the theorem. In the proof, we reduce 2-CSP instance $\Gamma$ to a D-NNT instance with a neuron for each variable and an output vertex for each constraint of $\Gamma$; thus, the width is in the worst case the number of constraints. We encode weights in the network such that a constraint $e$ is satisfied  if and only if the output corresponding to $e$ is within a restricted range of values. 
	
	\subsubsection{Proof of Theorem~\ref{thm:ETH}:}
	
	We give a reduction from 2-CSP to D-NNT on a two-layer network. Let $\Gamma = (G = (V,E_G),\Sigma,C)$ be a 2-CSP instance. Let $|\Sigma| = n$ be the size of the alphabet and let $|E| = k$ be the number of constraints of $\Gamma$. Let $g:\Sigma \rightarrow \{1,\ldots, n\}$ be an arbitrary bijection. Construct the following D-NNT instance $I = (N = (V = \{s\} \cup H \cup T,E_N), \mathcal{D}, (\sigma^v)_{v \in V \setminus \{s\}}, \mathcal{L}, \Theta = (W_e, B_e)_{e \in E}, \gamma)$.

	The network has one hidden layer $H = \{h_v \mid v \in V\}$, with one neuron for each vertex of $G$. There are $k$ output vertices $T = \{t_e \mid e \in E\}$, one for each constraint. Additionally, the edges are $E_N =  \{(s,h_v) \mid v \in V\} \cup \{(h_v, t_e) \mid v \in \textnormal{inc}(e), e \in E\}$, where $ \textnormal{inc}(e)$ are the vertices incident to an edge $e \in E$. The dataset is $\cD = \{(x,y)\}$, such that $x = 1$ and $y = (1)_{e \in E}$. For each neuron $h \in H$, the activation function is the identity function. The loss function on some input $(f_e)_{e \in E}$ is defined by $\cL((f_e)_{e \in E},y) = 1$ if for all $e \in E$ it holds that  $\left(g^{-1}\left(\ell_e \right), g^{-1}\left(m_e\right)\right) \in~C_{e}$, where we define $\ell_e = f_e \mod 2 \cdot n$ and $m_e = \frac{f_e- \ell_e}{2 \cdot n}$.  Otherwise, the value of the loss function is $\cL((f_e)_{e \in E},y) = 2$. 
	
	For each edge in the second layer of the form $e^u_{(u,v)} = (h_{u},t_{(u,v)})$, $u,v \in V$, define the weight and bias spaces $W_{e^u_{(u,v)}} = \{1\}$ and $B_{e^u_{(u,v)}} = \{0\}$. Moreover, for each edge $e^v_{(u,v)} = (h_{v},t_{(u,v)})$, $u,v \in V$, define $W_{e^v_{(u,v)}} = \{2 \cdot n\}$ and $B_{e_{(u,v)}} = \{0\}$. For each edge in the first layer $e_v = (s,h_v)$, $v \in V$, define $W_{e_v} = \{1,\ldots,n\}$ and $B_{e_v} = \{0\}$. Finally, define the error parameter as  $\gamma = 1$. Note that the encoding size of the reduced D-NNT instance is mostly affected by the size of the alphabet assuming that $n \gg k$.

	The running time of the reduction is polynomial in the encoding size of $\Gamma$. It remains to prove correctness. Let $\phi: V \rightarrow \Sigma$ be a feasible assignment for $\Gamma$. We find a set of weights that train the network below the threshold. For each $v \in V$, select the weight $w_{(s,h_v)} = g(\phi(v))$ for the edge $(s,h_v)$. Since all other edges have a single option for a weight (similarly, bias), this gives a unique selection $\theta = (w_e, b_e)_{e \in E_N}$ of parameters for all edges. 
	
	For the following, fix some $e = (u,v) \in E$. Let $f_e$ be the value obtained by vertex $t_e$ using the model $f_{\theta}(x)$. By the definition of $\theta$, it holds that $f_e = g(\phi(u))+2 \cdot n \cdot g(\phi(v))$. Let $\ell_e = f_e \mod 2 \cdot n$ and $m_e = \frac{f_e- \ell_e}{2 \cdot n}$. Then, it holds that $\ell = g(\phi(u))$ and $m_e = g(\phi(v))$; thus, $g^{-1}\left(\ell_e \right) = \phi(u)$ and $g^{-1}\left(m_e\right) = \phi(v)$. Since $\phi$ is a feasible assignment to $\Gamma$, it follows that $(\phi(u),\phi(v)) \in C_{e}$. Thus, $\left(g^{-1}\left(\ell_e \right), g^{-1}\left(m_e\right)\right) \in C_{e}$. As the above holds for any edge $e \in E$, we conclude that $\cL(f_{\theta}(x),y) = \cL((f_e)_{e \in E},y) = 1 = \gamma$.  
	
	For the second direction, let $\theta = (w_e,b_e)_{e \in E_N}$ be  a set of parameters that train the network to error bounded by the threshold $\cL(f_{\theta}(x),y) = \cL((f_e)_{e \in E},y) \leq \gamma$. 
	Define an assignment $\phi: V \rightarrow \Sigma$ by $\phi(v) = w_{(s,h_v)}$. Since $w_{(s,h_v)} \in \{1,\ldots, n\}$ for all $v \in V$ it holds that $\phi$ is indeed an assignment for $\Gamma$. It remains to show that $\phi$ is a feasible assignment. Fix some $e = (u,v) \in E$. Let $f_e$ be the value obtained by vertex $t_e$ using the model $f_{\theta}(x)$, let $\ell_e = f_e \mod 2 \cdot n$, and let $m_e = \frac{f_e- \ell_e}{2 \cdot n}$. Since the parameters $\theta$ satisfy that  $\cL(f_{\theta}(x),y) = \cL((f_e)_{e \in E},y) \leq \gamma$, by the definition of the loss function it follows that $\left(g^{-1}\left(\ell_e \right), g^{-1}\left(m_e\right)\right) \in C_{e}$ for all $e \in E$. By the definition of the weight spaces, it must hold that $\phi(u) = \ell_e$ and $\phi(v) = m_e$. Hence, $(\phi(u), \phi(v)) \in C_e$ for all $e = (u,v) \in E$, implying that $\phi$ is a feasible assignment. This gives the two directions of the reduction. 
	
	Based on the reduction, we can prove the theorem. Let $\mathcal{I}$ be the family of D-NNT instances on a two-layer network with the identity activation function, input dimension $d = 1$, and weights and biases only in $\{0,1\}$.
	Assume there exists a computable function $f$ and an algorithm that decides every $I \in \mathcal{I}$ in time $f(k) \cdot |I|^{o \left(\frac{k}{\log k}\right)}$, where $k$ is the width of the network. Then, by the above reduction, such an algorithm implies an algorithm for 2-CSP in time $f(k) \cdot n^{o \left(\frac{k}{\log k}\right)}$, where $k,n$ are the number of constraints and the size of the alphabet, respectively. Thus, assuming ETH, by Theorem~\ref{thm:Chen} such an algorithm does not exist which concludes the proof. \qed

	To prove Theorem~\ref{thm:SetCover}, we use the classic {\em exact set cover} problem. In the exact set cover problem, we are given a universe $U = \{u_1,\ldots, u_n\}$ of elements, a collection of sets $\mathcal{S} \subseteq 2^U$, and a cardinality bound $K \in \mathbb{N}$. The goal is to decide if there are $K$ sets $S_1, \ldots, S_K \in \mathcal{S}$ such that for all $u \in U$ there is {\em exactly} one set $S_i, 1\leq i \leq K$, such that $u \in S_i$.   
	In our reduction, we create a simple D-NNT instance with a single input, output, and neuron. Each data point $(x_i,y_i)$ corresponds to an element in the exact set cover instance. In addition, we have an extra data point used to guarantee a selection of at most $K$ sets. The dimensions correspond to sets, so the $j$-th entry in $x_i$ is $1$ if and only if element $i$ is in dimension $j$. The weights of the edges are also binary and correspond to a selection of sets.

	\subsubsection{Proof of Theorem 1.5:}
	
	We give a reduction from exact set cover. Let $\mathcal{C} = (U,\mathcal{S}, K)$ be a set cover instance with $n$ elements $U = \{u_1,\ldots, u_n\}$, $m$ sets $\mathcal{S} = \{S_1,\ldots, S_m\}$, and a cardinality bound $K$. Based on $\mathcal{C}$, we construct the following D-NNT instance $I = (N = (V = \{s\} \cup H \cup T,E), \mathcal{D}, (\sigma^v)_{v \in V \setminus \{s,t\}}, \mathcal{L}, \Theta = (W_e, B_e)_{e \in E}, \gamma)$. 
	
	\begin{itemize}
		\item The network has one hidden layer $H = \{h\}$ with a single neuron and one output. That is, $N = (\{s,h,t\}, E)$, where $E = \{(s,h), (h,t)\}$. 
		
		\item The data set is $\cD = \{(x_i,y_i)\}_{i \in \{0,\ldots, n\}}$. For all $i \in \{1,\ldots, n\}$ define $y_i = 1$ and $x_i \in \{0,1\}^m$ where $x_i[j] = 1$ if and only if $u_i \in S_j$, for all $j \in \{1,\ldots, m\}$. In addition, define $x_0 = (1)_{j \in \{1,\ldots, m\}}$ and $y_0 = K$.  
		
		\item The activation function of $h$ is the identity function.  
		
		
		\item The loss function is $\cL(f_{\theta}(x),y) = (f_{\theta}(x)-y)^2$. 
		
		\item For the edge $e = (s,h)$ and any dimension $j \in \{1,\ldots, m\}$ define the weight space $W_{e}[j] = \{0,1\}$ and let $B_e = \{0\}$ be the bias space.  
		
		\item For the edge $e = (h,t)$ define the weight space $W_{e} = \{1\}$ and let $B_e = \{0\}$ be the bias space. 
		
		\item Define the error parameter as $\gamma = 0$. 
	\end{itemize}
	
	The running time can be computed in polynomial time. We prove that there is a solution for $\mathcal{C}$ if and only if the network can be trained with zero error. For the first direction, let $S^* = \{S_{i_1}, \ldots, S_{i_K}\} \subseteq \mathcal{S}$ be a solution for $\mathcal{C}$. For each $j \in \{1,\ldots, m\}$ define $w_{(s,h)}[j] = 1$ if and only if $S_j \in S^*$. We remark that this uniquely defines a set of parameters $\theta$ for the network. Then, since $\left|S^*\right| = K$ it follows that $w_{(s,h)} \cdot x_0 = \sum_{j \in \{1,\ldots, m\}} w_{(s,h)}[j] \cdot 1 = K$; hence, $f_{\theta}(x_0) = y_0 = K$ implying that $\cL(f_{\theta}(x_0), y_0) = 0$. In addition, as $S^*$ is a solution for $\mathcal{C}$, for each $i \in \{1,\ldots, n\}$ there is exactly one $j \in \{1,\ldots, m\}$ such that $u_i \in S_j$ and $S_j \in S^*$. Thus, $w_{(s,h)} \cdot x_i =  \sum_{j \in \{1,\ldots, m\}} w_{(s,h)}[j] \cdot 1 = 1$ and $f_{\theta}(x_i) = y_i = 1$, implying that $\cL(f_{\theta}(x_i), y_i) = 0$. Overall, we conclude that $\sum_{i \in \{0,1,\ldots, n\}} \cL(f_{\theta}(x_i), y_i) = 0 = \gamma$. 
	
	For the second direction of the reduction, let $\theta = (w_{(s,h)}, b_{(s,h)}, w_{(h,t)}, b_{(h,t)})\in \Theta$ such that $\sum_{i \in \{0,1,\ldots, n\}} \cL(f_{\theta}(x_i), y_i) \leq \gamma$. Define $S^*  = \{S_j \in \mathcal{S} \mid w_{(s,h)}[j] = 1\}$ as all sets corresponding to {\em active} entries of the weight vector $w_{(s,h)}$. We show below that $S^*$ is a solution for $\mathcal{C}$.
	Since $\gamma = 0$ and as $\cL$ is the standard sum of the squares, we conclude that for all $i \in \{0,1,\ldots, n\}$ it holds that $\cL(f_{\theta}(x_i), y_i) = 0$. In particular, as $\cL(f_{\theta}(x_0), y_0) = (f_{\theta}(x_0)-K)^2 = 0$, we conclude that $|S^*| = K$ by the definition of $x_0$.  It remains to show that each element is covered exactly once. Let $i \in \{1,\ldots, n\}$; since $\cL(f_{\theta}(x_i), y_i) = (f_{\theta}(x_i)-1)^2= 0$, it follows that there is exactly one $j \in \{1,\ldots, m\}$ such that $u_i \in S_j$ and $w_{(s,h)}[j] = 1$. Hence, there is exactly one $S \in S^*$ such that $u_i \in S$. It follows that $S^*$ is a feasible solution for $\mathcal{C}$. The proof follows from the above reduction using the classic NP-Hardness of exact set cover \cite{gary1979computers}. \qed

	\section{A Pseudo-Polynomial Algorithm} 
	\label{sec:algorithm}
	
	In this section, we describe a pseudo-polynomial algorithm for D-NNT on a two-layer network as described in Theorem~\ref{thm:algorithm}. The proofs from this section are deferred to the supplementary material.  For the remainder of this section, fix a D-NNT instance $I = (N = (V = \{s\} \cup H \cup \{t\},E), \mathcal{D}, (\sigma^h)_{h \in H}, \mathcal{L}, \Theta = (W_e, B_e)_{e \in E}, \gamma)$ with one hidden layer $H$, one input vertex $s$, one output $t$, with activation functions only on the hidden layer.  
	
	Let $d$ be the input dimension (with output dimension being one), let $H = \{h_1,\ldots, h_k\}$ be the neurons in the hidden layer, and let $n_0 = |\cD|$ be the size of the dataset. Let $e_1,\ldots, e_k$ be the edges connecting the input vertex $s$ to the $k$ neurons. With a slight abuse of notation, for all $a \in \{1,\ldots, k\}$ let $W_{e_a}[d+1] = B_{e_a}$; in addition, for all $i \in \{1,\ldots, n_0\}$ let $x_i[d+1] = 1$ (these notations slightly simplify our claims). We remark that the network is not necessarily fully-connected and that the activation and loss functions can be arbitrary.

	As we consider a pseudo-polynomial time algorithm, we may assume without the loss of generality that   $W_e, B_e \subset \mathbb{N}$ for every $e \in E$ (otherwise, simply scale all numbers in the input). Let $W_{\max}$ be the largest number that can be computed via the network with or without applying the activation and loss functions. Moreover, let $M = d \cdot W_{\max} \cdot k$. Observe that $M$ is pseudo-polynomial in the input size $|I|$.
	
	We define a dynamic program $\textnormal{\textsf{DimDP}}$ with an entry for each number $1,\ldots,M$, representing the value obtained at the input to the activation function on each neuron $q \in \{1,\ldots, k\}$, for each data point $(x,y) \in \cD$, and dimensions $1,\ldots, j$ in the $q$-th neuron.
	The entries  describe selections of prefixes of the parameters that compute a specific vector of outputs, one output for each data point. For simplicity, we first describe a simplified version of the dynamic program, and of a second dynamic program that will be introduced later, returning only boolean values; we later expand these programs to return selected parameters themselves using simple backtracking.  
	
	Formally, define $\textnormal{\textsf{DimDP}}: \{0,\ldots, M\}^{n_0} \times \{1,\ldots, d+1\} \times \{1,\ldots, k\} \rightarrow \{\textnormal{\textsf{true}}, \textnormal{\textsf{false}}\}$ inductively as follows. Fix some $q \in \{1,\ldots, k\}$; the definition of the dynamic program is analogous to all such values of $q$. For each $m \in \{0,\ldots, M\}^{n_0}$, define $\textnormal{\textsf{DimDP}}[m,q,1] = \textnormal{\textsf{true}}$ if and only if there is $w \in W_{e_q}[1]$ such that $w \cdot x_i[1] = m[i]$ for all $i \in \{1,\ldots, n_0\}$; otherwise, set $\textnormal{\textsf{DimDP}}[m,q,1] = \textnormal{\textsf{false}}$. Then, for all $j \in \{2,\ldots,d+1\}$ and $m \in \{0,\ldots, M\}^{n_0}$ define $\textnormal{\textsf{DimDP}}[m,q,j] = \textnormal{\textsf{true}}$ if and only if there are $w \in W_{e_q} [j]$ and $m_1, m_2 \in \{0,\ldots, M\}^{n_0}$ such that $\textnormal{\textsf{DimDP}}[m_1,q,j-1] = \textnormal{\textsf{true}}$, $m_1+m_2 = m$, and for all $i \in \{1,\ldots, n_0\}$ it holds that $w \cdot x_i[j] = m_2[i]$.
	The following result summarizes the correctness of $\textnormal{\textsf{DimDP}}$.
	
	\begin{lemma}
		\label{lem:DP}
		For all $m \in \{0,\ldots, M\}^{n_0}$, $q \in \{1,\ldots, k\}$, and $j \in \{1\ldots, d+1\}$ it holds that $\textnormal{\textsf{DimDP}}[m,q,j] = \textnormal{\textsf{true}}$ if and only if there are $\left(w[r] \in W_{e_q}[r]\right)_{r \in \{1,\ldots,j\}}$ such that for all $i \in \{1,\ldots, n_0\}$ it holds that $\sum^{j}_{r = 1} x_i[r] \cdot w[r] = m[i]$.
	\end{lemma}

	\begin{proof}
		We prove the lemma separately for each $q \in \{1,\ldots, k\}$ by induction on $j$.  For the base case, where $j = 1$, by the definition of the dynamic program, for all $m \in \{0,\ldots, M\}^{n_0}$ it holds that $\textnormal{\textsf{DimDP}}[m,q,1] = \textnormal{\textsf{true}}$ if and only if there is $w \in W_{e_q}[1]$ such that $w \cdot x_i[1] = m[i]$ for all $i \in \{1,\ldots, n_0\}$. Assume that the claim holds for $j-1$, where $j \in \{2,\ldots,d+1\}$ and consider the claim for $j$. By the definition of $\textnormal{\textsf{DimDP}}$,
		for all $m \in \{0,\ldots, M\}^{n_0}$ it holds that $\textnormal{\textsf{DimDP}}[m,q,j] = \textnormal{\textsf{true}}$ if and only if there are $w \in W_{e_q} [j]$ and $m_1, m_2 \in \{0,\ldots, M\}^{n_0}$ such that $\textnormal{\textsf{DimDP}}[m_1,q,j-1] = \textnormal{\textsf{true}}$, $m_1+m_2 = m$, and for all $i \in \{1,\ldots, n_0\}$ it holds that $w \cdot x_i[j] = m_2[i]$. Using the induction hypothesis, $\textnormal{\textsf{DimDP}}[m_1,q,j-1] = \textnormal{\textsf{true}}$ if and only if there are $\left(w[r] \in W_{e_q}[r]\right)_{r \in \{1,\ldots,j-1\}}$ such that for all $i \in \{1,\ldots, n_0\}$ it holds that $\sum^{j-1}_{r = 1} x_i[r] \cdot w[r] = m_1[i]$. Thus, overall we proved that $\textnormal{\textsf{DimDP}}[m,q,j] = \textnormal{\textsf{true}}$ if and only if there are $\left(w[r] \in W_{e_q}[r]\right)_{r \in \{1,\ldots,j\}}$ such that for all $i \in \{1,\ldots, n_0\}$ it holds that $\sum^{j}_{r = 1} x_i[r] \cdot w[r] = m[i]$.
	\end{proof}

	Once we have computed the above dynamic program $\textnormal{\textsf{DimDP}}$, we compute a second dynamic program $\textnormal{\textsf{FinalDP}}: \{0,\ldots, M\}^{n_0} \times \{0,\ldots, k\} \rightarrow \{\textnormal{\textsf{true}}, \textnormal{\textsf{false}}\}$, where entry $\textnormal{\textsf{FinalDP}}[m,q]$ describes whether the total weight arriving to the output on input $x_i$ only from neurons $h_1,\ldots, h_q$ is precisely $m[i]$.  Formally, for all $m \in \{0,\ldots, M\}^{n_0}$ define $\textnormal{\textsf{FinalDP}}[m,0] = \textnormal{\textsf{true}}$ if and only if $m = (0)_{i \in \{1,\ldots, n_0\}}$.  For all $q \in \{1,\ldots,k\}$ and $m \in \{0,\ldots, M\}^{n_0}$ define $\textnormal{\textsf{FinalDP}}[m,q] = \textnormal{\textsf{true}}$ if and only if there are $m_1, m_2, m_3 \in \{0,\ldots, M\}^{n_0}$ such that $\textnormal{\textsf{FinalDP}}[m_1,q-1] = \textnormal{\textsf{true}}$, $m_1+m_3 = m$, $\textnormal{\textsf{DimDP}}[m_2,q,d+1] = \textnormal{\textsf{true}}$, and $m_3 = \sigma^{h_q} \left( m_2 \right)$. The following result summarizes the correctness of $\textsf{FinalDP}$.
	
	\begin{lemma}
		\label{lem:Final}
		For all $m \in \{0,\ldots, M\}^{n_0}$ and $q \in \{0,\ldots, k\}$ it holds that $\textnormal{\textsf{FinalDP}}[m,q] = \textnormal{\textsf{true}}$ if and only if there are $\left(w_{e_a}[r] \in W_{e_a}[r]\right)_{r \in \{1,\ldots,d\}, a \in \{1,\ldots, q\}}$ and $\left(b_{e_a} \in B_{e_a} \right)_{a \in \{1,\ldots, q\}}$ such that for all $i \in \{1,\ldots, n_0\}$ it holds that $\sum^{q}_{a = 1} \sigma^{h_a} \left(\sum^{d}_{r = 1} x_i[r] \cdot w_{e_a}[r] +b_{e_a}\right) = m[i]$. 
	\end{lemma}

	\begin{proof}
			We prove the lemma by induction on $q \in \{0,\ldots, k\}$. For the base case, where $q = 0$, by the definition of the dynamic program, for all $m \in \{0,\ldots, M\}^{n_0}$ it holds that $\textnormal{\textsf{FinalDP}}[m,q] = \textnormal{\textsf{true}}$ if and only if $m$ is the zero vector. Thus, the base case follows by the  definition of the dynamic program. 
		
		Assume that the claim holds for some $q-1$, where $q \in \{1,\ldots,k\}$ and consider the claim for $q$. By the definition of $\textnormal{\textsf{FinalDP}}$,
		for all $m \in \{0,\ldots, M\}^{n_0}$ it holds that $\textnormal{\textsf{FinalDP}}[m,q,j] = \textnormal{\textsf{true}}$ if and only if there are $m_1, m_2, m_3 \in \{0,\ldots, M\}^{n_0}$ such that $\textnormal{\textsf{FinalDP}}[m_1,q-1] = \textnormal{\textsf{true}}$, $m_1+m_3 = m$, $\textnormal{\textsf{DimDP}}[m_2,q,d+1] = \textnormal{\textsf{true}}$, and $m_3 = \sigma^{h_q} \left( m_2 \right)$. By the induction hypothesis, $\textnormal{\textsf{FinalDP}}[m_1,q-1] = \textnormal{\textsf{true}}$ if and only if there are $\left(w_{e_a}[r] \in W_{e_a}[r]\right)_{r \in \{1,\ldots,d\}, a \in \{1,\ldots, q-1\}}$ and $\left(b_{e_a} \in B_{e_a} \right)_{a \in \{1,\ldots, q-1\}}$ such that for all $i \in \{1,\ldots, n_0\}$ it holds that $\sum^{q-1}_{a = 1} \sigma^{h_a} \left(\sum^{d}_{r = 1} x_i[r] \cdot w_{e_a}[r] +b_{e_a}\right) = m_1[i]$. Moreover, by  Lemma 5.1 it holds that $\textnormal{\textsf{DimDP}}[m_2,q,d+1] = \textnormal{\textsf{true}}$ if and only if there are $\left(w[r] \in W_{e_q}[r]\right)_{r \in \{1,\ldots,d+1\}}$ such that for all $i \in \{1,\ldots, n_0\}$ it holds that $\sum^{d+1}_{r = 1} x_i[r] \cdot w[r] = m_2[i]$ (recall that we use $x_i[d+1] = 1$ and $W_{e_q}[d+1] = B_{e_q}$). Thus, for the above parameters, $\textnormal{\textsf{DimDP}}[m_2,q,d+1] = \textnormal{\textsf{true}}$ if and only if for all $i \in \{1,\ldots, n_0\}$ it holds that $\sigma^{h_q} \left(\sum^{d+1}_{r = 1} x_i[r] \cdot w[r] \right) = m_3$. 
		
		Combined, the above implies that $\textnormal{\textsf{FinalDP}}[m,q] = \textnormal{\textsf{true}}$ if and only if there are $\left(w_{e_a}[r] \in W_{e_a}[r]\right)_{r \in \{1,\ldots,d\}, a \in \{1,\ldots, q\}}$ and $\left(b_{e_a} \in B_{e_a} \right)_{a \in \{1,\ldots, q\}}$ such that for all $i \in \{1,\ldots, n_0\}$ it holds that $\sum^{q}_{a = 1} \sigma^{h_a} \left(\sum^{d}_{r = 1} x_i[r] \cdot w_{e_a}[r] +b_{e_a}\right) = m[i]$. This gives the statement of the lemma. 
	\end{proof}

	We slightly alter the above dynamic programs so that instead of only returning \textnormal{\textsf{true}} or \textnormal{\textsf{false}}, they also retrieve an actual set of parameters corresponding to the entry. Namely, for all $m \in \{0,\ldots, M\}^{n_0}$, $q \in \{1,\ldots, k\}$, and $j \in \{1,\ldots, d+1\}$ the dynamic program $\textnormal{\textsf{DimDP}}[m,q,j]$ keeps a set of parameters $\left(w[r] \in W_{e_q}[r]\right)_{r \in \{1,\ldots,j\}}$ such that for all $i \in \{1,\ldots, n_0\}$ it holds that $\sum^{j}_{r = 1} x_i[r] \cdot w[r] = m[i]$ (recall that $x_i[d+1] = 1$). Similarly, for all $m \in \{0,\ldots, M\}^{n_0}$ and $q \in \{0,\ldots, k\}$ the dynamic program $\textnormal{\textsf{FinalDP}}[m,q]$ keeps a set of parameters $\left(w_{e_a}[r] \in W_{e_a}[r]\right)_{r \in \{1,\ldots,d\}, a \in \{1,\ldots, q\}}$ and $\left(b_{e_a} \in B_{e_a} \right)_{a \in \{1,\ldots, q\}}$ such that for all $i \in \{1,\ldots, n_0\}$ it holds that $\sum^{q}_{a = 1} \sigma^{h_a} \left(\sum^{d}_{r = 1} x_i[r] \cdot w_{e_a}[r] +b_{e_a}\right) = m[i]$. The parameters can be selected efficiently using standard backtracking.  
	
	We now define a model based on the above dynamic programs. Select a set of parameters $\theta \in \Theta$ such that 
	$\sum_{i=1}^{n} \cL\left(f_{\theta}(x_i), y_i\right)$ is minimized over all possible parameter sets corresponding to $\textnormal{\textsf{FinalDP}}[m,k]$, for all $m \in \{0,\ldots, M\}^{n_0}$ such that $\textnormal{\textsf{FinalDP}}[m,k] = \textsf{true}$. Using the above, we obtain an algorithm deciding $I$ by checking whether the training loss over the selected model surpasses the threshold. The pseudocode of the algorithm is given in Algorithm~\ref{alg:algorithm}. The next lemma analyzes the running time and correctness of the algorithm.

	\begin{algorithm}[tb]
		\caption{Dynamic Programming Algorithm}
		\label{alg:algorithm}
		\textbf{Input}: a D-NNT instance $I$\\
		\textbf{Output}: {\textnormal{\textsf{true}} if $I$ is a Yes-instance}
		\begin{algorithmic}[1] 
			\STATE Compute the dynamic program $\textnormal{\textsf{DimDP}}$. 
			\STATE Compute the dynamic program $\textnormal{\textsf{FinalDP}}$.
			\STATE Compute the model $f_{\theta}$ based on $\textnormal{\textsf{DimDP}}$ and $\textnormal{\textsf{FinalDP}}$. 
			\STATE \textbf{return true} if and only if $\sum_{i=1}^{n_0} \cL\left(f_{\theta}(x_i), y_i\right) \leq \gamma$.  
		\end{algorithmic}
	\end{algorithm}
	
	\begin{lemma}
		\label{lem:alg}
		Algorithm~\ref{alg:algorithm} correctly decides the \textnormal{D-NNT} instance and can be computed in time $M^{O(n_0)} \cdot d \cdot k$.
	\end{lemma}
	
	\begin{proof}
			We start by showing the correctness of the algorithm; this follows easily by the definition of the dynamic programs computed in the course of the algorithm. Assume that $I$ is a yes-instance. That is, there are parameters $\theta^* = (w^*_{e_a}, b^*_{e_a})_{a \in \{1,\ldots,k\}} \in \Theta$ such that $\sum_{i=1}^{n_0} \cL\left(f_{\theta^*}(x_i), y_i\right) \leq \gamma$. Therefore, there is $m \in \{0,\ldots, M\}^{n_0}$ such that for all $i \in \{1,\ldots, n_0\}$ it holds that $\sum^{k}_{a = 1} \sigma^{h_a} \left(\sum^{d}_{r = 1} x_i[r] \cdot w^*_{e_a}[r]+b^*_{e_a}\right) = m[i]$.  Thus, by Lemma 5.2, it follows that $\textnormal{\textsf{FinalDP}}[m,k] = \textsf{true}$. Consequently, the dynamic program $\textnormal{\textsf{FinalDP}}[m,k]$ saves parameters $(w_{e_a}, b_{e_a})_{a \in \{1,\ldots,k\}} \in \Theta$ such that for all $i \in \{1,\ldots, n_0\}$ it holds that $\sum^{k}_{a = 1} \sigma^{h_a} \left(\sum^{d}_{r = 1} x_i[r] \cdot w_{e_a}[r]+b_{e_a}\right) = m[i]$. Therefore, by the selection of parameters $\theta$ for the model, it follows that $\sum_{i=1}^{n_0} \cL\left(f_{\theta}(x_i), y_i\right) \leq \sum_{i=1}^{n_0} \cL\left(f_{\theta^*}(x_i), y_i\right)  \leq \gamma$. Thus, the algorithm returns that $I$ is a yes-instance as required. 
		
		For the second direction, assume that the algorithm returns that $I$ is a yes-instance. Thus, by the definition of the algorithm, $\sum_{i=1}^{n_0} \cL\left(f_{\theta}(x_i), y_i\right) \leq \gamma$, where $\theta = (w_{e_a}, b_{e_a})_{a \in \{1,\ldots,k\}} \in \Theta$ are the selected parameters by the algorithm. Since the algorithm chooses a feasible set of parameters, it follows that $I$ is indeed a yes-instance. 
		
		It remains to analyze the running time. Constructing each entry of the dynamic programs \textnormal{\textsf{DimDP}} and \textnormal{\textsf{FinalDP}} can be done in time $M^{O(n_0)}$. Since the number of entries is in both programs is bounded by $M^{n_0} \cdot k \cdot (d+1)+M^{n_0} \cdot (k+1) = M^{O(n_0)} \cdot |I|$, the overall time required for computing these programs and finding the optimal set of parameters $\theta$ accordingly takes time $M^{O(n_0)} \cdot |I|$. By the above, the proof follows. 
	\end{proof}

	The above directly implies the proof of the theorem. 
	
	\subsubsection{Proof of Theorem~\ref{alg:algorithm}:}
	
	Follows from Lemma~\ref{lem:alg}. \qed

	\section{Proof of Theorem 1.2}
	\label{sec:proof1.2}

	We give below the proof of Theorem 1.2. It is based on a reduction that augments the output space. The only assumptions we make on the input D-NNT instance are that it has one source vertex (though with an unbounded dimension) and that the activation functions map $0$ to $0$ (e.g., ReLU).  
	
	\subsubsection{Proof of Theorem 1.2:} Consider some D-NNT instance $I = (N = (V = \{s\} \cup H \cup T,E), \mathcal{D}, (\sigma^v)_{v \in V \setminus \{s\}}, \mathcal{L}, \Theta = (W_e, B_e)_{e \in E}, \gamma)$. We construct as follows a reduced C-NNT instance $I' = (N = (V,E), \mathcal{D}', (\sigma'^v)_{v \in V \setminus \{s\}}, \mathcal{L}', \Theta', \gamma)$. Let $d$ be the number of input dimensions of $I$. With a slight abuse of notation, let $W_{e}[d+1] = B_e$ for all $e \in E$. Additionally, for edges $e = (u,v), u \neq s$, we assume, for the simplicity of the notation, that $W_{e}$ is still $(d+1)$-dimensional, but all dimensions $2,\ldots, d$ are always zero, i.e., $W_e[i] = \{0\}$ for all $i \in \{2,\ldots,d\}$. We also assume that the activation functions satisfy that  for all $v \in V \setminus \{s\}$, it holds that $\sigma^v(0) = 0$.  
	
	Assume without the loss of generality that all weights and biases given in $I$ are integers (this can be guaranteed by a simple scaling since $I$ is a D-NNT instance whose parameter space is given entirely in the input). Let $M_0$ be the maximum absolute value of a number in the input of $I$ and let $M = 2 \cdot M_0$. Let $f: E \times \{1,\ldots, d+1\} \rightarrow \mathbb{N}$ be some injective function such that (i) for all $e \in E$ and $i \in \{1,\ldots, d+1\}$ it holds that $f(e,i) > 10^M$ and (ii) for all $e_1,e_2 \in E$ and $i_1,i_2 \in \{1,\ldots, d+1\}$ it holds that $|f(e_1,i_1)-f(e_2,i_2)| > 10^{M}$. For some $e \in E$, and $i \in \{1,\ldots, d+1\}$, let $F(e,i) = [f(e,i), M \cdot f(e,i)] \cup [-f(e,i) \cdot M, -f(e,i)]$. By the definition of $f$, note that $(F(e,i))_{e \in E, i \in \{1,\ldots, d+1\}}$ are disjoint. Moreover, for each $e = (u,v) \in E$, let $P_e$ be some path from $s$ to $u$ in $N$.  
	
	We define below the instance $I'$. The output vertices are all vertices in $V$.  The input dimensions are $\{1,\ldots,d+2\}$. Define $\cD' = \{(x\oplus 0 ,y) \mid (x,y) \in \cD\} \cup \cD''$, where $\cD'' = \{(x_{e,i}, y_{e,i}) \mid e \in E, i \in \{1,\ldots, d+1\}\}$ and $\oplus$ denotes concatenation with an additional entry (containing zero). For each $e = (s,v) \in E$ and $i \in \{1,\ldots, d+1\}$ define $x_{e,i} = \chi_{e,i}$ where $\chi_{e,i}$ is a vector of dimensions $\{1,\ldots, d+2\}$ with zeros in all entries except for $f(e,i)$ in the $i$-th entry. In addition, for each $e = (u,v) \in E, u \neq s$, and $i \in \{1,\ldots, d+1\}$, define $x_{e,i}$ as the vector of dimensions $\{1,\ldots, d+2\}$ with zeros in all entries except for $f(e,i)$ in the $(d+2)$-th entry. 
	
	For each $v \in V \setminus \{s\}$ define the activation function $\sigma^{'v}$ by  $\sigma^{'v}(z) = \sigma^{v}(z)$ if $|z| < 10^{M}$, by $\sigma^{'v}(z) = 0$ if $v \notin P_{e}$ where $z \in F(e,i)$ for some $e \in E$ and $i \in \{1,\ldots, d+1\}$; otherwise, define $\sigma^{'v}(z) = z$. 
	
	
	
	The loss function $\cL'$ is defined as $\cL'(z,y) = \cL(x,y)$ for all $y$ such that $(x,y) \in \cD$ for some $x$. In addition, for any $e = (u,v) \in E$ and $i \in \{1,\ldots, d+1\}$ define $\cL'(z,y_{e,i}) = 1$, where $z = (z^q)_{q \in V}$, if (i) for all $q \in P_e$ (including $u$) it holds that $z^q = f(e,i)$, (ii) for all $q \in V \setminus (\{s\} \cup P_e)$ such that there is a directed path from $q$ to $v$ it holds that $z^q = 0$, and (iii) $\frac{z^v}{f(e,i)} \in W_e[i]$.  Otherwise, define $\cL'(z,y_{e,i}) = (2+\gamma) \cdot |V| \cdot |E| \cdot (d+1) \cdot |\cD|$.  
	
	The parameter set is $\Theta' = (W'_e,B'_e)_{e \in E}$, where $W'_e[i],B'_e = \mathbb{R}$, for all $e \in E$ and $i \in \{1,\ldots,d+2\}$ (as we construct a continuous NNT instance $I'$). Finally, define the error parameter as $\gamma' = \gamma+|E| \cdot (d+1)$.
	
	Observe that the reduction can be computed in time $\textnormal{poly}(|I|)$. We prove below the two directions of the reduction. Assume that $I$ is a yes-instance, i.e., there are parameters $\theta = (W_e,B_e) \in \Theta$ such that  $\sum_{(x,y) \in \cD} \cL\left(f_{\theta}(x), y\right) \leq \gamma$. Define the following parameters for $I'$. For any $e = (u,v) \in E$ and $i \in \{1,\ldots, d+1\}$ define $w'_e[i] = w_e[i]$. Moreover, define $w_{e}[d+2] = 1$. Let $f'_{\theta'}$ be the selected model. By the definition of the network, for any $(x,y) \in \cD$ the maximum number obtained using our defined weights on the input $x$ is bounded by $M_0$. Hence, by the definition of the activation functions, the computation is identical to the computation of $x$ on $I$. Thus, we conclude that $\sum_{(x,y) \in \cD} \cL'\left(f'_{\theta'}(x), y\right) = \sum_{(x,y) \in \cD} \cL\left(f_{\theta}(x), y\right) \leq \gamma$. 
	
	It remains to prove a bound on the rest of the dataset $\cD'$. Let $e = (u,v) \in E$ and $i \in \{1,\ldots,d+1\}$. Consider two cases. If $u = s$, then $x_{e,i} = \chi_{e,i}$; hence, $w'_{e}[i] \cdot x_{e,i} = w_e[i] \cdot f(e,i) = \sigma^{'v} \left( w_e[i] \cdot f(e,i) \right) = z^v(x_{e,i})$, where the last inequality holds by the definition of $\sigma^{'v}$ since $w_e[i] \cdot f(e,i) \in F(e,i) \cup \{0\}$. This implies that $w'_e[i] = \frac{z^v(x_{e,i})}{f(e,i)} \in W_e[i]$ since $w_e[i] \in W_e[i]$.   
	
	Otherwise, assume that $u \neq s$. Let $v_1,\ldots, v_C$ be a topological order of all vertices in $V \setminus \{s\}$ that have a directed path to $v$. We prove that for all $c = \{1,\ldots, C\}$ (i.e., including $u$ but excluding $v$), it holds that $z^q(x_{e,i}) = f(e,i)$ if $q \in P_e$ and $z^q(x_{e,i}) = 0$ otherwise. 
	The proof is by induction on $c$.   
	
	For the base case, let $q = v_1$. Since we consider a topological order, it follows that $(s,q) \in E$. Thus, it holds that $z^q(x_{e,i}) = \sigma^{'q} \left(x_{e,i} \cdot w_{(s,q)} \right) = x_{e,i}[d+2] \cdot 1 = f(e,i)$ if $q \in P_e$ and $z^q(x_{e,i}) = 0$ otherwise.
	Now, assume that the claim holds for $v_1,\ldots, v_{c-1} \in V$ for some $c \in \{2,\ldots, C\}$. Let $q \in V$ such that $q = v_c$. 
	If $(s,q) \in E$, the proof is analogous to the base case. Thus, assume that $(s,q) \notin E$.
	By the induction hypothesis and since we consider a topological sort, for all $q' \in V \setminus \{s\}$ such that $(q',q) \in E$ it holds that either $z^{q'}(x_{e,i}) = f(e,i)$ if $q' \in P_e$ or that $z^{q'}(x_{e,i}) = 0$ if $q' \notin P_e$. Thus,
	since $z^{q}(x_{e,i}) = \sigma^{q} \left( \sum_{(q',q) \in E} z^{q'}(x_{e,i})\right)$, we have the following cases. If $q \in P_e$, there is exactly one $q' \in P_e$ such that $(q',q) \in E$; thus, $z^{q'}(x_{e,i}) = f(e,i)$ and $z^{q''}(x_{e,i}) = 0$ for all $q'' \in V \setminus \{q'\}$ such that $(q'',q) \in E$ by the induction hypothesis.  It follows that $\sum_{(q',q) \in E} z^{q'}(x_{e,i}) = f(e,i)$ which implies that $z^{q}(x_{e,i}) = f(e,i)$ by the definition of the activation function. Otherwise,  $q \notin P_e$; in this case, there is at most one $q' \in P_e$ such that $(q',q) \in E$; thus, $z^{q'}(x_{e,i}) = f(e,i)$ and $z^{q''}(x_{e,i}) = 0$ for all $q'' \in V \setminus \{q'\}$ such that $(q'',q) \in E$ by the induction hypothesis. It follows that $\sum_{(q',q) \in E} z^{q'}(x_{e,i}) \in \{0, f(e,i)\}$, which implies that $z^{q}(x_{e,i}) = 0$ since $q \notin P_e$. 
	
	
	
	Therefore, by the above proof it holds that $z^{u}(x_{e,i}) =  f(e,i)$ in particular and that $z^{q}(x_{e,i}) =  0$ for all $q \in V$ such that $(q,v) \in E$. Thus, $w'_{e}[i] \cdot z^u(x_{e,i}) = w_e[i] \cdot f(e,i) = z^v(x_{e,i})$ by the definition of the activation function. Thus, $\frac{z^v(x_{e,i})}{f(e,i)} \in W_e[i]$ since $w_e[i] \in W_e[i]$. Therefore, by the definition of the loss function $\cL'$, it holds that $\cL'(f'_{\theta'}(x_{e,i}), y_{e,i}) = 1$. Hence,
	
	\begin{equation*}
		\begin{aligned}
			&\sum_{(x,y) \in \cD'} \cL'\left(f'_{\theta'}(x), y\right) 
			\\={} & \sum_{(x,y) \in \cD} \cL\left(f_{\theta}(x), y\right) + \sum_{(x_{e,i},y_{e,i}) \in \cD''} \cL\left(f_{\theta}(x), y\right)
			\\\leq{} & \gamma + |E| \cdot (d+1)
			\\= {} & \gamma'
		\end{aligned}
	\end{equation*}
	
	For the second direction of the reduction, let $\theta' = (w'_e,b'_e)_{e \in E} \in \Theta'$ be parameters for $I'$ such that $\sum_{(x,y) \in \cD'} \cL'\left(f'_{\theta'}(x), y\right) \leq \gamma'$. Define $\theta = (w_e,b_e) \in \Theta$ where $b_e = b'_e$ and $w_e[i] = w'_e[i]$ for all $e \in E$ and $i \in \{1,\ldots, d+1\}$. Let $f_{\theta}$ be the constructed model. By the definition of $\cL'$, the above implies that $\cL'(f'_{\theta'}(x_{e,i}), y_{e,i}) = 1$ for all $e = (u,v) \in E$ and $i \in \{1,\ldots, d+1\}$. Thus, again by the definition of $\cL'$, for all $e = (u,v) \in E$ and $i \in \{1,\ldots, d+1\}$ the following holds. (i) For all $q \in P_e$ (including $u$) it holds that $z^q(x_{e,i}) = f(e,i)$, (ii) for all $q \in V \setminus (\{s\} \cup P_e)$ such that there is a directed path from $q$ to $v$ it holds that $z^q(x_{e,i}) = 0$, and (iii) $\frac{z^v(x_{e,i})}{f(e,i)} \in W_e[i]$.
	

	Thus, by (i) and (ii) it holds that $\frac{z^v(x_{e,i})}{f(e,i)} = w'_e[i]$; hence, by (iii) we have
	$w_e[i] = w'_e[i] \in W_e[i]$ (recall that $w_e[d+1] = b_e$ hence this implies $b_e \in B_e$ as well).  Hence, $\theta \in \Theta$. 
	
	Finally, note that $\sum_{(x,y) \in \cD''} \cL'\left(f'_{\theta'}(x), y\right) \geq |E| \cdot(d+1)$.  Thus, since $\sum_{(x,y) \in \cD'} \cL'\left(f'_{\theta'}(x), y\right) \leq \gamma'$ it follows that $\sum_{(x,y) \in \cD} \cL'\left(f'_{\theta'}(x), y\right) \leq \gamma$. Therefore, it follows that $\sum_{(x,y) \in \cD} \cL\left(f_{\theta}(x), y\right) \leq \gamma$ by the definition of the loss function $\cL'$ and since $\theta \in \Theta$. \qed
	
	\bibliography{aaai25}

\begin{thebibliography}{43}
\providecommand{\natexlab}[1]{#1}

\bibitem[{Abrahamsen, Kleist, and Miltzow(2021)}]{NEURIPS2021_9813b270}
Abrahamsen, M.; Kleist, L.; and Miltzow, T. 2021.
\newblock Training Neural Networks is ER-complete.
\newblock In Ranzato, M.; Beygelzimer, A.; Dauphin, Y.; Liang, P.; and Vaughan,
  J.~W., eds., \emph{Advances in Neural Information Processing Systems},
  volume~34, 18293--18306. Curran Associates, Inc.

\bibitem[{Allender et~al.(2009)Allender, B{\"{u}}rgisser,
  Kjeldgaard{-}Pedersen, and Miltersen}]{DBLP:journals/siamcomp/AllenderBKM09}
Allender, E.; B{\"{u}}rgisser, P.; Kjeldgaard{-}Pedersen, J.; and Miltersen,
  P.~B. 2009.
\newblock On the Complexity of Numerical Analysis.
\newblock \emph{{SIAM} J. Comput.}, 38(5): 1987--2006.

\bibitem[{Amizadeh, Matusevych, and Weimer(2019)}]{amizadeh2019learning}
Amizadeh, S.; Matusevych, S.; and Weimer, M. 2019.
\newblock Learning to solve circuit-sat: An unsupervised differentiable
  approach.
\newblock In \emph{International Conference on Learning Representations}.

\bibitem[{Arora et~al.(2018)Arora, Basu, Mianjy, and Mukherjee}]{AroraBMM18}
Arora, R.; Basu, A.; Mianjy, P.; and Mukherjee, A. 2018.
\newblock Understanding Deep Neural Networks with Rectified Linear Units.
\newblock In \emph{6th International Conference on Learning Representations,
  {ICLR} 2018, Vancouver, BC, Canada, April 30 - May 3, 2018, Conference Track
  Proceedings}. OpenReview.net.

\bibitem[{Baldassi and Braunstein(2015)}]{baldassi2015max}
Baldassi, C.; and Braunstein, A. 2015.
\newblock A max-sum algorithm for training discrete neural networks.
\newblock \emph{Journal of Statistical Mechanics: Theory and Experiment},
  2015(8): P08008.

\bibitem[{Bertschinger et~al.(2023)Bertschinger, Hertrich, Jungeblut, Miltzow,
  and Weber}]{BertschingerHJM23}
Bertschinger, D.; Hertrich, C.; Jungeblut, P.; Miltzow, T.; and Weber, S. 2023.
\newblock Training Fully Connected Neural Networks is {\(\exists\)}R-Complete.
\newblock In \emph{Advances in Neural Information Processing Systems 36: Annual
  Conference on Neural Information Processing Systems 2023, NeurIPS 2023, New
  Orleans, LA, USA, December 10 - 16, 2023}.

\bibitem[{Blum and Rivest(1988)}]{blum1988training}
Blum, A.; and Rivest, R. 1988.
\newblock Training a 3-node neural network is NP-complete.
\newblock \emph{Advances in neural information processing systems}, 1.

\bibitem[{Boob, Dey, and Lan(2022)}]{Boob22}
Boob, D.; Dey, S.~S.; and Lan, G. 2022.
\newblock Complexity of training ReLU neural network.
\newblock \emph{Discret. Optim.}, 44(Part): 100620.

\bibitem[{B{\"{u}}rgisser and Jindal(2024)}]{DBLP:conf/soda/BurgisserJ24}
B{\"{u}}rgisser, P.; and Jindal, G. 2024.
\newblock On the Hardness of PosSLP.
\newblock In Woodruff, D.~P., ed., \emph{Proceedings of the 2024 {ACM-SIAM}
  Symposium on Discrete Algorithms, {SODA} 2024, Alexandria, VA, USA, January
  7-10, 2024}, 1872--1886. {SIAM}.

\bibitem[{Courbariaux, Bengio, and David(2015)}]{courbariaux2015binaryconnect}
Courbariaux, M.; Bengio, Y.; and David, J.-P. 2015.
\newblock Binaryconnect: Training deep neural networks with binary weights
  during propagations.
\newblock \emph{Advances in neural information processing systems}, 28.

\bibitem[{Dey, Wang, and Xie(2020)}]{dey20}
Dey, S.~S.; Wang, G.; and Xie, Y. 2020.
\newblock Approximation Algorithms for Training One-Node ReLU Neural Networks.
\newblock \emph{{IEEE} Trans. Signal Process.}, 68: 6696--6706.

\bibitem[{Dinur and Steurer(2014)}]{dinur2014analytical}
Dinur, I.; and Steurer, D. 2014.
\newblock Analytical approach to parallel repetition.
\newblock In \emph{Proceedings of the forty-sixth annual ACM symposium on
  Theory of computing}, 624--633.

\bibitem[{Dutta, Saxena, and Sinhababu(2022)}]{10.1145/3510359}
Dutta, P.; Saxena, N.; and Sinhababu, A. 2022.
\newblock Discovering the Roots: Uniform Closure Results for Algebraic Classes
  Under Factoring.
\newblock \emph{J. ACM}, 69(3).

\bibitem[{Froese and Hertrich(2024)}]{Froese2024}
Froese, V.; and Hertrich, C. 2024.
\newblock Training neural networks is np-hard in fixed dimension.
\newblock \emph{Advances in Neural Information Processing Systems}, 36.

\bibitem[{Froese, Hertrich, and
  Niedermeier(2022{\natexlab{a}})}]{froese2022computational}
Froese, V.; Hertrich, C.; and Niedermeier, R. 2022{\natexlab{a}}.
\newblock The computational complexity of ReLU network training parameterized
  by data dimensionality.
\newblock \emph{Journal of Artificial Intelligence Research}, 74: 1775--1790.

\bibitem[{Froese, Hertrich, and Niedermeier(2022{\natexlab{b}})}]{Froese22}
Froese, V.; Hertrich, C.; and Niedermeier, R. 2022{\natexlab{b}}.
\newblock The Computational Complexity of ReLU Network Training Parameterized
  by Data Dimensionality.
\newblock \emph{J. Artif. Intell. Res.}, 74: 1775--1790.

\bibitem[{Gary and Johnson(1979)}]{gary1979computers}
Gary, M.~R.; and Johnson, D.~S. 1979.
\newblock Computers and Intractability: A Guide to the Theory of
  NP-completeness.

\bibitem[{Goel et~al.(2017)Goel, Kanade, Klivans, and Thaler}]{GoelKKT17}
Goel, S.; Kanade, V.; Klivans, A.~R.; and Thaler, J. 2017.
\newblock Reliably Learning the ReLU in Polynomial Time.
\newblock In Kale, S.; and Shamir, O., eds., \emph{Proceedings of the 30th
  Conference on Learning Theory, {COLT} 2017, Amsterdam, The Netherlands, 7-10
  July 2017}, volume~65 of \emph{Proceedings of Machine Learning Research},
  1004--1042. {PMLR}.

\bibitem[{Goel et~al.(2021)Goel, Klivans, Manurangsi, and Reichman}]{Goel21}
Goel, S.; Klivans, A.~R.; Manurangsi, P.; and Reichman, D. 2021.
\newblock Tight Hardness Results for Training Depth-2 ReLU Networks.
\newblock In Lee, J.~R., ed., \emph{12th Innovations in Theoretical Computer
  Science Conference, {ITCS} 2021, January 6-8, 2021, Virtual Conference},
  volume 185 of \emph{LIPIcs}, 22:1--22:14. Schloss Dagstuhl - Leibniz-Zentrum
  f{\"{u}}r Informatik.

\bibitem[{Goodfellow, Bengio, and Courville(2016)}]{DBLP:books/daglib/0040158}
Goodfellow, I.~J.; Bengio, Y.; and Courville, A.~C. 2016.
\newblock \emph{Deep Learning}.
\newblock Adaptive computation and machine learning. {MIT} Press.

\bibitem[{Hayer and Bhalla(2005)}]{hayer2005molecular}
Hayer, A.; and Bhalla, U.~S. 2005.
\newblock Molecular switches at the synapse emerge from receptor and kinase
  traffic.
\newblock \emph{PLoS computational biology}, 1(2): e20.

\bibitem[{Hochreiter(1998)}]{hochreiter1998vanishing}
Hochreiter, S. 1998.
\newblock The vanishing gradient problem during learning recurrent neural nets
  and problem solutions.
\newblock \emph{International Journal of Uncertainty, Fuzziness and
  Knowledge-Based Systems}, 6(02): 107--116.

\bibitem[{Hu, Xiao, and Pennington(2020)}]{DBLP:conf/iclr/HuXP20}
Hu, W.; Xiao, L.; and Pennington, J. 2020.
\newblock Provable Benefit of Orthogonal Initialization in Optimizing Deep
  Linear Networks.
\newblock In \emph{8th International Conference on Learning Representations,
  {ICLR} 2020, Addis Ababa, Ethiopia, April 26-30, 2020}. OpenReview.net.

\bibitem[{Impagliazzo and Paturi(2001)}]{impagliazzo2001complexity}
Impagliazzo, R.; and Paturi, R. 2001.
\newblock On the complexity of k-SAT.
\newblock \emph{Journal of Computer and System Sciences}, 62(2): 367--375.

\bibitem[{Judd(1988)}]{judd1988complexity}
Judd, S. 1988.
\newblock On the complexity of loading shallow neural networks.
\newblock \emph{Journal of Complexity}, 4(3): 177--192.

\bibitem[{Khalife, Cheng, and Basu(2024)}]{khalife2024neural}
Khalife, S.; Cheng, H.; and Basu, A. 2024.
\newblock Neural networks with linear threshold activations: structure and
  algorithms.
\newblock \emph{Mathematical Programming}, 206(1): 333--356.

\bibitem[{Kleinberg and Tardos(2006)}]{books0015106}
Kleinberg, J.~M.; and Tardos, {\'{E}}. 2006.
\newblock \emph{Algorithm design}.
\newblock Addison-Wesley.
\newblock ISBN 978-0-321-37291-8.

\bibitem[{Li and Hao(2018)}]{li2018optimal}
Li, Q.; and Hao, S. 2018.
\newblock An optimal control approach to deep learning and applications to
  discrete-weight neural networks.
\newblock In \emph{International Conference on Machine Learning}, 2985--2994.
  PMLR.

\bibitem[{Lu et~al.(2003)Lu, Wang, Cheng, and Huang}]{lu2003circuit}
Lu, F.; Wang, L.-C.; Cheng, K.-T.; and Huang, R.-Y. 2003.
\newblock A circuit SAT solver with signal correlation guided learning.
\newblock In \emph{2003 Design, Automation and Test in Europe Conference and
  Exhibition}, 892--897. IEEE.

\bibitem[{Marx(2007)}]{marx2007can}
Marx, D. 2007.
\newblock Can you beat treewidth?
\newblock In \emph{48th Annual IEEE Symposium on Foundations of Computer
  Science (FOCS'07)}, 169--179. IEEE.

\bibitem[{Megiddo(1988)}]{megiddo1988complexity}
Megiddo, N. 1988.
\newblock On the complexity of polyhedral separability.
\newblock \emph{Discrete \& Computational Geometry}, 3: 325--337.

\bibitem[{Miller et~al.(2005)Miller, Zhabotinsky, Lisman, and
  Wang}]{miller2005stability}
Miller, P.; Zhabotinsky, A.~M.; Lisman, J.~E.; and Wang, X.-J. 2005.
\newblock The stability of a stochastic CaMKII switch: dependence on the number
  of enzyme molecules and protein turnover.
\newblock \emph{PLoS biology}, 3(4): e107.

\bibitem[{Pilanci and Ergen(2020)}]{PilanciE20}
Pilanci, M.; and Ergen, T. 2020.
\newblock Neural Networks are Convex Regularizers: Exact Polynomial-time Convex
  Optimization Formulations for Two-layer Networks.
\newblock In \emph{Proceedings of the 37th International Conference on Machine
  Learning, {ICML} 2020, 13-18 July 2020, Virtual Event}, volume 119 of
  \emph{Proceedings of Machine Learning Research}, 7695--7705. {PMLR}.

\bibitem[{Rastegari et~al.(2016)Rastegari, Ordonez, Redmon, and
  Farhadi}]{rastegari2016xnor}
Rastegari, M.; Ordonez, V.; Redmon, J.; and Farhadi, A. 2016.
\newblock Xnor-net: Imagenet classification using binary convolutional neural
  networks.
\newblock In \emph{European conference on computer vision}, 525--542. Springer.

\bibitem[{Raz and Safra(1997)}]{raz1997sub}
Raz, R.; and Safra, S. 1997.
\newblock A sub-constant error-probability low-degree test, and a sub-constant
  error-probability PCP characterization of NP.
\newblock In \emph{Proceedings of the twenty-ninth annual ACM symposium on
  Theory of computing}, 475--484.

\bibitem[{Rokh, Azarpeyvand, and Khanteymoori(2023)}]{rokh2023comprehensive}
Rokh, B.; Azarpeyvand, A.; and Khanteymoori, A. 2023.
\newblock A comprehensive survey on model quantization for deep neural networks
  in image classification.
\newblock \emph{ACM Transactions on Intelligent Systems and Technology}, 14(6):
  1--50.

\bibitem[{Schaefer(2009)}]{DBLP:conf/gd/Schaefer09}
Schaefer, M. 2009.
\newblock Complexity of Some Geometric and Topological Problems.
\newblock In Eppstein, D.; and Gansner, E.~R., eds., \emph{Graph Drawing, 17th
  International Symposium, {GD} 2009, Chicago, IL, USA, September 22-25, 2009.
  Revised Papers}, volume 5849 of \emph{Lecture Notes in Computer Science},
  334--344. Springer.

\bibitem[{Shalev-Shwartz and Ben-David(2014)}]{shalev2014understanding}
Shalev-Shwartz, S.; and Ben-David, S. 2014.
\newblock \emph{Understanding machine learning: From theory to algorithms}.
\newblock Cambridge university press.

\bibitem[{Sun et~al.(2022)Sun, Lao, Sundaramoorthi, and
  Yezzi}]{sun2022surprising}
Sun, Y.; Lao, D.; Sundaramoorthi, G.; and Yezzi, A. 2022.
\newblock Surprising instabilities in training deep networks and a theoretical
  analysis.
\newblock \emph{Advances in Neural Information Processing Systems}, 35:
  19567--19578.

\bibitem[{Zhang et~al.(2021)Zhang, Bengio, Hardt, Recht, and
  Vinyals}]{zhang2021understanding}
Zhang, C.; Bengio, S.; Hardt, M.; Recht, B.; and Vinyals, O. 2021.
\newblock Understanding deep learning (still) requires rethinking
  generalization.
\newblock \emph{Communications of the ACM}, 64(3): 107--115.

\bibitem[{Zheng et~al.(2016)Zheng, Song, Leung, and
  Goodfellow}]{zheng2016improving}
Zheng, S.; Song, Y.; Leung, T.; and Goodfellow, I. 2016.
\newblock Improving the robustness of deep neural networks via stability
  training.
\newblock In \emph{Proceedings of the ieee conference on computer vision and
  pattern recognition}, 4480--4488.

\bibitem[{Zhou et~al.(2018)Zhou, Moosavi-Dezfooli, Cheung, and
  Frossard}]{zhou2018adaptive}
Zhou, Y.; Moosavi-Dezfooli, S.-M.; Cheung, N.-M.; and Frossard, P. 2018.
\newblock Adaptive quantization for deep neural network.
\newblock In \emph{Proceedings of the AAAI Conference on Artificial
  Intelligence}, volume~32.

\bibitem[{Zhu et~al.(2016)Zhu, Han, Mao, and Dally}]{zhu2016trained}
Zhu, C.; Han, S.; Mao, H.; and Dally, W.~J. 2016.
\newblock Trained ternary quantization.
\newblock \emph{arXiv preprint arXiv:1612.01064}.

\end{thebibliography}
\end{document}